\documentclass{article}

\usepackage[preprint]{neurips_2025}
\makeatletter\renewcommand{\@notice}{}\makeatother
\AtBeginDocument{\newgeometry{textwidth=6.5in,textheight=9in,top=1in,headheight=12pt,headsep=25pt,footskip=30pt}}

\usepackage[utf8]{inputenc}
\usepackage[T1]{fontenc}
\usepackage{hyperref}
\usepackage{url}
\usepackage{booktabs}
\usepackage{amsfonts}
\usepackage{amsmath}
\usepackage{amsthm}
\usepackage{amssymb}
\usepackage{nicefrac}
\usepackage{microtype}
\usepackage{xcolor}
\usepackage{graphicx}
\usepackage{enumerate}
\usepackage{algorithm}
\usepackage{algorithmic}
\usepackage{multirow}
\usepackage{float}

\definecolor{placeholdercolor}{RGB}{180,0,0}

\newtheorem{theorem}{Theorem}[section]
   \newtheorem{proposition}[theorem]{Proposition}
   \newtheorem{lemma}[theorem]{Lemma}
   \newtheorem{corollary}[theorem]{Corollary}
    \theoremstyle{remark}
    \newtheorem*{remarkinner}{Remark}
    \newenvironment{remark}{\medskip\begin{remarkinner}}{\end{remarkinner}}

\title{LionVote: Per-Layer Learning Rate  \\
Adaptation for Lion}

\author{%
  Kris Atallah \\
  New York University \\
  New York, NY, USA \\
  \texttt{kris.a@nyu.edu}}

\begin{document}
\maketitle

\begin{abstract}
Per-layer diagnostics reveal that, at the prescribed learning rate,
Lion's effective scale is
$2.6$--$2.8{\times}$ too high for attention and MLP parameters and
${\sim}2{\times}$ too high for normalisation layers on
ViT-Tiny/CIFAR-100; this $32\%$ cross-layer-type disparity cannot be
reproduced by a single global rate.  The measurement comes from LionVote, a
per-layer learning rate mechanism in which each parameter tensor
maintains a compound level, a persistent integer updated every $c$~epochs by two
diagnostics (gradient direction stability and momentum health) resolved
by a validation loss tiebreaker.  Voting thresholds derive from
geometric identities, the EMA time constant, and a noise-floor
estimate; cadence is bounded structurally and selected by ablation.
On ViT-Tiny/CIFAR-100, LionVote achieves $69.7\%$ top-1 accuracy vs.\
Lion's $69.0\%$ ($p < 0.02$, Welch's $t$-test) and AdamW's $68.8\%$.
Per-layer adaptation value depends on both architectural heterogeneity
and task; on uniform CNN architectures tuned SGD with cosine annealing
remains dominant, and on ViT architectures gains are task-dependent.
\end{abstract}

\section{Introduction}
\label{sec:introduction}
Zhao et al.~\citep{zhao2025deconstructing} show that applying per-layer adaptive preconditioning only to the last layer and LayerNorm parameters recovers most of Adam's advantage over SGD on autoregressive language models; this is evidence that layer types have different optimisation characteristics. Their analysis does not prescribe how much each type should diverge from a base rate, nor propose a mechanism to determine this during training.

Existing adaptive methods operate per-coordinate or per-layer but statelessly, and schedule-free approaches~\citep{defazio2024schedulefree} maintain a single global rate. No existing method combines stateful layer-level adaptation with data-driven rate adjustment (\S\ref{sec:related:gap}).

LionVote is a per-layer learning rate mechanism for Lion~\citep{chen2023symbolic}. Lion's sign operation discards gradient magnitude, preventing directionally unstable layers from compensating through larger updates; this makes per-layer miscalibration more consequential for sign-based methods than for second-moment methods like Adam (\S\ref{sec:analysis:layers}). Each parameter tensor maintains a compound level (a persistent integer, \S\ref{sec:method:lr}) that modulates the base learning rate exponentially. Every $c$~epochs, two per-layer diagnostics (gradient direction stability and momentum health, with derived thresholds; Appendix~\ref{app:vote1}--\ref{app:vote2}) vote on whether to increase, decrease, or maintain each layer's rate. When the votes conflict, a validation loss tiebreaker resolves the decision. The mechanism adds one extra accumulation per parameter per batch; voting runs once every $c$ epochs. With all votes disabled, compound levels decay to zero and LionVote reduces to standard Lion.

We evaluate LionVote on WideResNet and ViT-Tiny across CIFAR-10 and CIFAR-100, with 8 seeds per configuration. The contributions are:

\begin{enumerate}
\item A per-layer adaptive mechanism for a sign-based optimizer whose voting thresholds derive from geometric identities and the EMA time constant; cadence and structural parameters (the exponent divisor $d{=}2$, maximum level $L{=}4$) are bounded rather than uniquely determined (Appendix~\ref{app:vote1}--\ref{app:cadence}).  The threshold derivation methodology is reusable for principled design of per-layer mechanisms in other optimizers.

\item A quantified finding about Lion: its effective scale is $2.6$--$2.8{\times}$ too high for attention and MLP parameters and ${\sim}2{\times}$ too high for normalisation parameters on ViT-Tiny/CIFAR-100, measured via compound level trajectories across 8 seeds. Normalisation layers receive $32\%$ higher effective scale than attention layers (\S\ref{sec:analysis:layers}). Because the compound multiplier scales both the sign update and decoupled weight decay, this measures joint LR+WD miscalibration (\S\ref{sec:analysis:layers}).

\item Evidence that per-layer adaptation value depends on both architectural heterogeneity and task. On ViT-Tiny, compound level spread between layer types is $2.8{\times}$ larger than on WideResNet, and LionVote at cadence~8 achieves the best accuracy on CIFAR-100 ($69.7\%$). The same architecture and comparable spread on CIFAR-10 yields a significant loss vs.\ Lion ($-0.88$~pp, $p < 0.001$). On WideResNet, tuned SGD with cosine annealing remains dominant (\S\ref{sec:analysis:arch}).
\end{enumerate}

\section{Related Work}
\label{sec:related}
Deep networks are not monolithic. Residual networks \citep{he2016resnet}, wide residual networks \citep{zagoruyko2016wideresnet}, and vision transformers \citep{dosovitskiy2021vit} differ qualitatively in how gradients flow and how representations evolve across depth.  Standard schedules such as cosine annealing \citep{loshchilov2017sgdr} and step decay apply a single learning rate trajectory to every parameter, regardless of each layer's convergence state.

\subsection{Per-Parameter Adaptation}
\label{sec:related:perparam}

Per-coordinate adaptation progressed from AdaGrad's \citep{duchi2011adagrad} accumulated squared gradients through Adam \citep{kingma2014adam} and decoupled weight decay \citep{loshchilov2019adamw}; all adapt per-coordinate, not per-layer.  Sign-based methods (signSGD \citep{bernstein2018signsgd}, Lion \citep{chen2023symbolic}) reduce the update to a sign of a momentum--gradient interpolation, matching Adam at lower memory cost under a global schedule.

Cautious Optimizers~\citep{liang2026cautious} mask update coordinates by sign agreement, operating per-coordinate per-batch without per-layer history.

\subsection{Per-Layer Adaptation}
\label{sec:related:perlayer}

Howard and Ruder \citep{howard2018ulmfit} show that assigning progressively smaller learning rates to lower layers outperforms a single global rate when adapting pretrained language models to new tasks. LARS \citep{you2017lars} introduced per-layer rate scaling by multiplying each layer's learning rate by the ratio of its weight norm to its gradient norm, working well in large-batch settings. LAMB \citep{you2020lamb} applies the same trust-ratio idea to Adam, enabling large-batch training where standard Adam degrades. LARS and LAMB are stateless: a layer that has been stable for many epochs is treated identically to one emerging from a noisy phase.

Gradient direction stability---whether a layer's gradient aligns consistently across epochs---is a complementary signal to the weight-to-gradient ratio. To our knowledge, using alignment between successive gradient estimates to modulate step size has not been applied at the layer level.

Muon~\citep{jordan2024muon} applies orthogonalised Newton--Schulz updates to hidden layers but hardcodes which layers receive which treatment.

Zhao et al.~\citep{zhao2025deconstructing} show that restricting per-layer adaptive preconditioning to the last layer and LayerNorm parameters suffices to capture most of the accuracy benefit on autoregressive language models, without quantifying per-type divergence or proposing a discovery mechanism.

Hao et al.~\citep{hao2025lanton} (LANTON) assign noise-adaptive per-layer learning rates to geometry-aware optimizers such as Muon, estimating gradient variance in the dual norm per layer at each step.  The adaptation is instantaneous: it does not accumulate cross-epoch evidence or maintain persistent per-layer state.

\subsection{Data-Driven Schedule Replacement}
\label{sec:related:auto}

AutoDrop \citep{wang2024autodrop} monitors angular velocity to trigger rate reductions automatically. D-Adaptation \citep{defazio2023dadaptation} derives the learning rate from an online lower bound on the distance to a minimizer. Orvieto and Xiao \citep{orvieto2024ngn} exploit a Gauss-Newton reformulation to derive stepsizes that warm up, peak, and decay without a schedule. Defazio et al.\ \citep{defazio2024schedulefree} replace momentum with iterate interpolation and averaging, matching hand-tuned cosine schedules with no stopping time.

All of these methods maintain a single global rate for the entire parameter vector.

\subsection{The Gap}
\label{sec:related:gap}

Schedule-free methods establish that data-driven adaptation can replace hand-designed global schedules; layer-type-aware methods establish that layers benefit from heterogeneous treatment. No existing method bridges the two. LARS and LAMB adapt per-layer but statelessly (\S\ref{sec:related:perlayer}). Schedule-free and D-Adaptation methods accumulate history but apply a single rate globally. Muon differentiates layer types but hardcodes which layers receive which treatment. LANTON adapts per-layer rates instantaneously (\S\ref{sec:related:perlayer}) but does not accumulate cross-epoch evidence. Meta-learned optimizers~\citep{andrychowicz2016learning} can in principle discover per-layer policies, but sacrifice interpretability and require expensive meta-training on proxy tasks.

\section{Proposed Approach}
\label{sec:method}
LionVote is a per-layer learning rate mechanism for the Lion optimizer
\citep{chen2023symbolic}, implemented as a PyTorch \texttt{Optimizer}
subclass.  It makes no structural modification to Lion's update rule.
Each parameter tensor receives an effective learning rate
$\alpha_i$ adjusted by a persistent integer \emph{compound level} that
is updated every $c$ epochs through a two-vote system with a global
tiebreaker.

\subsection{Per-Layer Learning Rate}
\label{sec:method:lr}

Each parameter $i$ has a compound level
$s_i \in \{-L, \ldots, 0, \ldots, +L\}$ (default $L = 4$) that
modulates the base learning rate:
\begin{equation}
\label{eq:alpha}
  \alpha_i \;=\; \mathrm{lr} \cdot \exp\!\bigl(s_i \cdot \beta_1 / 2\bigr),
\end{equation}
where $\mathrm{lr}$ is the global base learning rate (which may be
updated by an external scheduler) and $\beta_1$ is Lion's
sign-interpolation coefficient.  All compound levels are initialised
to zero.  At the default $\beta_1 = 0.9$, each unit step in $s_i$
corresponds to a ${\times}1.57$ multiplier, and the full range
$[-4, +4]$ spans from ${\times}0.165$ to ${\times}6.05$ relative to
the base rate.  The derivation of the exponent $\beta_1/2$ and the
bound $L = 4$ are given in Appendix~\ref{app:exponent}
and~\ref{app:maxlevel}.

\subsection{Batch Update}
\label{sec:method:step}

On every training batch, LionVote performs the standard Lion update
with $\alpha_i$ in place of the global learning rate: decoupled weight
decay, followed by a sign update
$\operatorname{sign}(\beta_1 m_i + (1 - \beta_1) g_i)$, followed by a
$\beta_2$-EMA momentum update (Algorithm~\ref{alg:lionvote}).  The raw gradient is accumulated for
epoch-level voting.

\subsection{Epoch Step and Voting}
\label{sec:method:epoch}

At the end of every epoch, LionVote computes the epoch-mean gradient
$\bar{g}_i = a_i / N$ for each parameter, resets the accumulators,
and recomputes $\alpha_i$ from the current $\mathrm{lr}$ so that
per-layer rates track any external scheduler.  On
\emph{voting epochs}---epochs where
$\mathrm{epoch} \bmod c = 0$, with $c$ the voting cadence---the
system additionally evaluates two per-layer diagnostics and one global
tiebreaker.

\paragraph{Vote 1: Gradient Direction Stability.}
\[
  \mathrm{alignment}_i \;=\;
    \frac{\bar{g}_i^{(\mathrm{curr})} \cdot \bar{g}_i^{(\mathrm{prev})}}
         {\|\bar{g}_i^{(\mathrm{curr})}\| \,
          \|\bar{g}_i^{(\mathrm{prev})}\| + \epsilon}.
\]
The vote is $+1$ if $\mathrm{alignment}_i > 0.5$, $-1$ if
$\mathrm{alignment}_i < 0$, and $0$ otherwise ($\epsilon = 10^{-8}$
throughout).  The lower threshold
follows directly from the definition of cosine (negative iff angle
exceeds $\pi/2$); the upper threshold corresponds to a two-thirds
supermajority of coordinate sign pairs agreeing under coordinate
isotropy (Appendix~\ref{app:vote1}).  Vote~1 uses epoch-mean gradients rather
than momentum to avoid self-contamination from the EMA.

\paragraph{Vote 2: Momentum Health.}
\[
  r_i \;=\; \frac{\|m_i\|}{\|\bar{g}_i^{(\mathrm{curr})}\| + \epsilon}.
\]
The vote is $-1$ if $r_i > e$ (momentum norm exceeds gradient norm by
more than a factor of $e$), $+1$ if $r_i < 1/e$ (gradient norm
exceeds momentum norm by more than a factor of $e$), and $0$
otherwise.  Both thresholds derive from the EMA's intrinsic time
constant: $\beta_2^{\tau} = e^{-1}$ for any $\beta_2 \in (0,1)$
(Appendix~\ref{app:vote2}).

\paragraph{Tiebreaker: Validation Loss.}
The tiebreaker is a global scalar that fires only when the two local
votes sum to zero (both abstain or conflict):
\[
  \delta \;=\;
    \frac{\bar{L}_{\mathrm{prev}} - \bar{L}_{\mathrm{curr}}}
         {|\bar{L}_{\mathrm{prev}}| + \epsilon}.
\]
It returns $+1$ if $\delta > 0.01$, $-1$ if $\delta < -0.01$, and $0$
otherwise ($\epsilon = 10^{-8}$ as in Votes~1 and~2).  The $1\%$
threshold matches the noise floor of cross-entropy loss estimation for
typical validation set sizes (Appendix~\ref{app:tiebreaker}).
Under Lion on ViT-Tiny/CIFAR-100, the tiebreaker condition is
satisfied for ${\sim}59\%$ of per-parameter voting decisions
(Appendix~\ref{app:vote_dynamics}), so the tiebreaker is not a
tie-breaking edge case but the dominant signal path;
Vote~1 or Vote~2 alone determine the outcome for the remaining
${\sim}41\%$.

\paragraph{Vote Resolution.}
Vote~1 monitors directional stability; Vote~2 monitors the
momentum-gradient magnitude ratio.  The final vote is resolved
per the table in Algorithm~\ref{alg:lionvote}.

\subsection{Compound Level Update}
\label{sec:method:update}

The compound level $s_i$ is updated via the \emph{asymmetric} rule:
vote $= +1$ resets $s_i$ to $0$ if negative, else increments
(capped at $+L$); vote $= -1$ resets to $0$ if positive, else
decrements (capped at $-L$); vote $= 0$ decays toward zero.
A single opposing vote revokes the accumulated level entirely.

\subsection{Algorithm}
\label{sec:method:algorithm}

\begin{figure}[t]
\small
\begin{minipage}[t]{0.62\textwidth}
\begin{algorithm}[H]
\caption{LionVote}
\label{alg:lionvote}
\begin{algorithmic}[1]
\REQUIRE $\theta$, lr, $\beta_1, \beta_2, \lambda$, cadence $c$, max level $L$
\STATE Init $m_i, a_i, s_i \leftarrow 0$;
  $\alpha_i \leftarrow \mathrm{lr}$; $N \leftarrow 0$
\FOR{epoch $= 1, 2, \ldots$}
  \FOR{each batch with gradient $g_i$}
    \STATE $\theta_i \leftarrow \theta_i
      - \alpha_i \lambda \theta_i$
      \hfill $\triangleright$ decay
    \STATE $\theta_i \leftarrow \theta_i
      - \alpha_i \operatorname{sign}(\beta_1 m_i
        + (1{-}\beta_1)g_i)$
      \hfill $\triangleright$ Lion
    \STATE $m_i \leftarrow \beta_2 m_i + (1{-}\beta_2)g_i$
    \STATE $a_i \leftarrow a_i + g_i$; $N \leftarrow N{+}1$
  \ENDFOR
  \STATE $\bar{g}_i \leftarrow a_i / N$
  \IF{ep $\bmod c = 0$ \AND $\bar{g}^{(\mathrm{prev})}$ exists}
    \STATE Tiebreaker from $\bar{L}_{\mathrm{curr}}, \bar{L}_{\mathrm{prev}}$
    \FOR{each parameter $i$}
      \STATE $v_1 \leftarrow$ Vote~1; $v_2 \leftarrow$ Vote~2
      \STATE Resolve $\rightarrow v$; update $s_i$
    \ENDFOR
  \ENDIF
  \IF{ep $\bmod c = 0$ \OR ep $= 1$}
    \STATE Snapshot $\bar{g}_i^{(\mathrm{prev})}, \bar{L}_{\mathrm{prev}}$
  \ENDIF
  \STATE $\alpha_i \leftarrow \mathrm{lr} \cdot
    \exp(s_i \cdot \beta_1/2)$
  \STATE $a_i, N \leftarrow 0$
\ENDFOR
\end{algorithmic}
\end{algorithm}
\end{minipage}\hfill
\begin{minipage}[t]{0.35\textwidth}
\vspace{1.8em}
\centering
\small
\textbf{Vote Resolution}\\[0.5em]
\begin{tabular}{ccc}
\toprule
V1 & V2 & Result \\
\midrule
$\pm 1$ & same & agree \\
$\pm 1$ & $0$ & V1 \\
$0$ & $\pm 1$ & V2 \\
\multicolumn{2}{c}{else} & tiebreaker \\
\bottomrule
\end{tabular}
\end{minipage}
\end{figure}

Computational cost is detailed in
Appendix~\ref{app:exp:computational_cost}.

\section{Experiments}
\label{sec:experiments}
We evaluate LionVote against AdamW, SGD with Nesterov momentum, and
Lion on four model--dataset configurations.  All experiments use 8
random seeds.  We report top-1 validation accuracy (mean $\pm$ std)
and epochs to reach intermediate accuracy thresholds.

\subsection{Experimental Setup}
\label{sec:experiments:setup}

\paragraph{Models and datasets.}
WideResNet-28-10~\citep{zagoruyko2016wideresnet} on CIFAR-10
(${\sim}36.5$M parameters, 200 epochs, batch size 128) and
WideResNet-40-10 on CIFAR-100 (${\sim}55.8$M parameters, 200 epochs,
batch size 128).  ViT-Tiny~\citep{dosovitskiy2021vit} (patch size~4,
embedding dimension~192, 12 blocks, 3 heads, ${\sim}5.7$M parameters)
on CIFAR-10 and CIFAR-100 (300 epochs, batch size 128).
WideResNet uses pre-activation blocks with BatchNorm and dropout~0.3.
ViT-Tiny uses learnable positional embeddings, stochastic depth
(linearly increasing to 0.1), and a CLS token classification head.

\paragraph{Augmentation.}
CNN configurations: RandomCrop, HorizontalFlip, and AutoAugment
(CIFAR-10 policy) for CIFAR-100.  ViT configurations: RandAugment$(2, 9)$,
RandomErasing$(p{=}0.25)$, Mixup$(\alpha{=}0.8)$/CutMix$(\alpha{=}1.0)$
(50/50 per batch), and label smoothing~0.1 with soft cross-entropy.

\paragraph{Optimizers.}
AdamW (lr${=}10^{-3}$), SGD (lr${=}0.1$, momentum 0.9, Nesterov),
Lion (lr${=}10^{-4}$, $\beta_1{=}0.9$, $\beta_2{=}0.99$), and
LionVote (same as Lion, plus voting).  Weight decay follows the Lion
paper~\citep{chen2023symbolic}: Lion and LionVote use $3$--$10{\times}$
larger WD than AdamW (e.g.\ $5{\times}10^{-1}$ vs.\ $5{\times}10^{-2}$
for ViT/CIFAR-10).  Full weight decay values per configuration are in
Appendix~\ref{app:exp:hparams}.  LARS and LAMB are omitted: both are
stateless norm-ratio methods designed for large-batch distributed
training (\S\ref{sec:related:perlayer}), a different regime from the
single-GPU setting studied here.

\paragraph{Schedulers.}
WRN configurations: cosine annealing.  ViT configurations: 5-epoch
linear warmup followed by cosine annealing.  LionVote is additionally
tested without any schedule.

\paragraph{Voting hyperparameters.}
Cadence $c \in \{4, 6, 8\}$, max level $L{=}4$, asymmetric update
rule.  All voting thresholds are derived
(Appendix~\ref{app:vote1}--\ref{app:tiebreaker}); cadence is selected by ablation within a structurally bounded range
(Appendix~\ref{app:cadence}).

\subsection{Main Results}
\label{sec:experiments:main}

\begin{table}[t]
\centering
\caption{Best top-1 validation accuracy (\%, mean $\pm$ std, 8 seeds).
WRN configurations use cosine scheduling; ViT configurations use cosine
with 5-epoch warmup.  Best result per row in \textbf{bold}.}
\label{tab:main}
\small
\begin{tabular}{l c c c c c}
\toprule
Configuration & SGD & AdamW & Lion & LionVote (c4) & LionVote (c8) \\
\midrule
WRN-28-10 / C10  & \textbf{96.08}{\scriptsize$\pm$0.15}
  & 93.80{\scriptsize$\pm$0.69}
  & 93.71{\scriptsize$\pm$0.40}
  & 93.14{\scriptsize$\pm$0.38}
  & 93.30{\scriptsize$\pm$0.82} \\
WRN-40-10 / C100 & \textbf{82.76}{\scriptsize$\pm$0.37}
  & 77.45{\scriptsize$\pm$0.44}
  & 76.53{\scriptsize$\pm$0.63}
  & 76.81{\scriptsize$\pm$1.26}
  & 75.81{\scriptsize$\pm$1.19} \\
ViT-Tiny / C10   & 69.20{\scriptsize$\pm$1.54}
  & 91.52{\scriptsize$\pm$0.21}
  & \textbf{92.40}{\scriptsize$\pm$0.10}
  & 90.80{\scriptsize$\pm$0.44}
  & 91.52{\scriptsize$\pm$0.18} \\
ViT-Tiny / C100  & 36.97{\scriptsize$\pm$5.86}
  & 68.75{\scriptsize$\pm$0.39}
  & 68.95{\scriptsize$\pm$0.41}
  & 68.90{\scriptsize$\pm$0.48}
  & \textbf{69.71}{\scriptsize$\pm$0.60} \\
\bottomrule
\end{tabular}
\end{table}

SGD with cosine annealing dominates the WRN configurations
(Table~\ref{tab:main}),
outperforming the next-best method by $2.3$~pp on CIFAR-10 and
$5.3$~pp on CIFAR-100, consistent with the established difficulty of beating tuned SGD
on uniform CNN architectures~\citep{schmidt2021crowded}.  On ViTs, Lion
outperforms both SGD and AdamW.  LionVote at
cadence~8 achieves the best result on ViT-Tiny/CIFAR-100 ($69.71\%$)
and ties AdamW on ViT-Tiny/CIFAR-10 ($91.52\%$).  The
ViT-Tiny/CIFAR-100 improvement over Lion is statistically significant
(Welch's $t$-test, $p = 0.017$, 8 seeds); cadence is selected by
ablation from $c \in \{4, 6, 8\}$
(\S\ref{sec:analysis:limitations}).  On ViT-Tiny/CIFAR-10,
Lion's advantage over LionVote at cadence~8 ($92.40\%$ vs.\
$91.52\%$) is also significant ($p < 0.001$).  At cadence~8, the
differences between Lion and LionVote are not significant on either
WRN configuration ($p > 0.15$).  At cadence~4, LionVote
significantly underperforms Lion on WRN-28-10/CIFAR-10 ($-0.57$~pp,
$p = 0.011$).

\begin{figure}[t]
\centering
\includegraphics[width=0.92\textwidth]{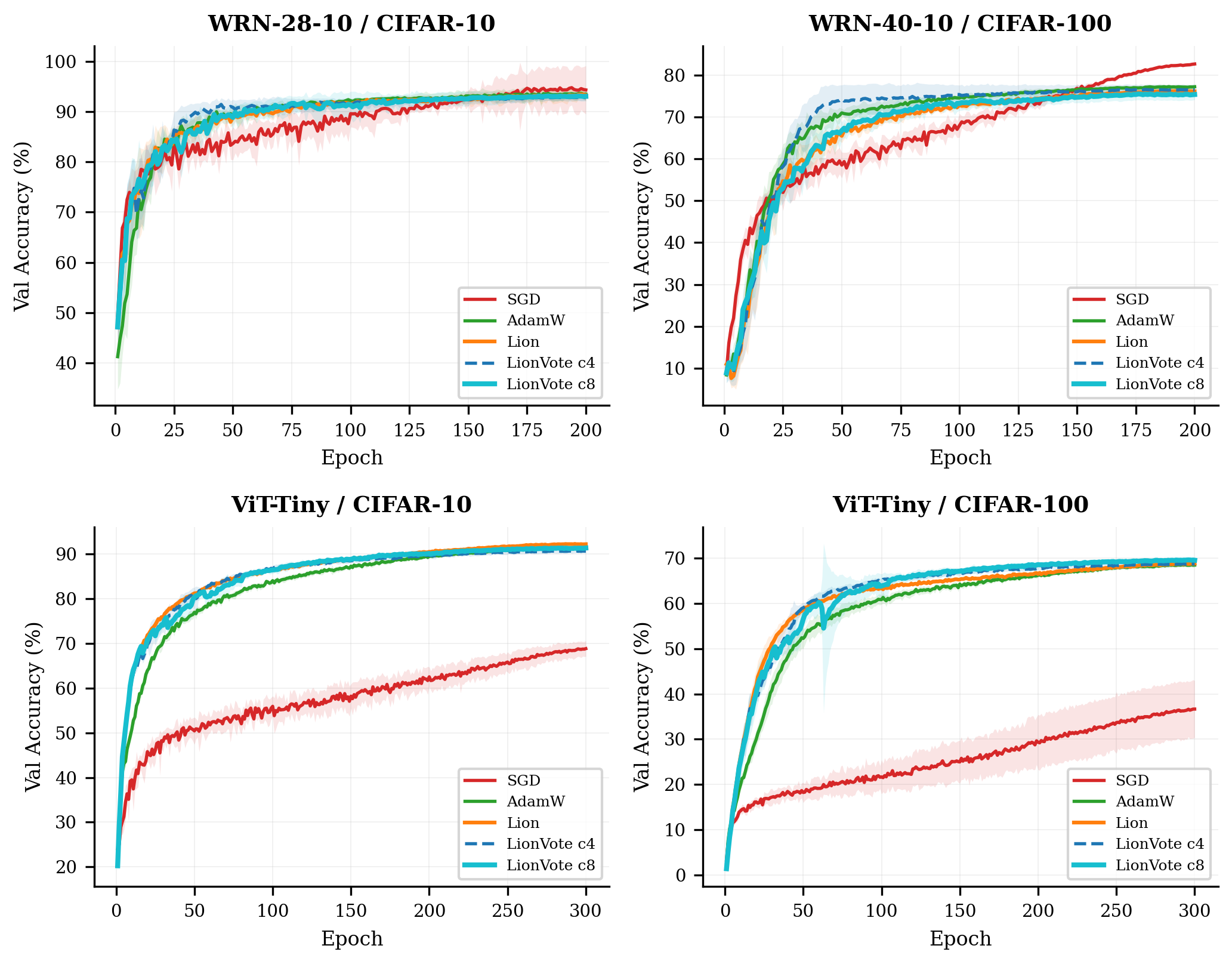}
\caption{Validation accuracy vs.\ epoch for all four configurations
(mean $\pm 1$ std over 8 seeds).  Panels (left to right, top to
bottom): WRN-28-10/C10, WRN-40-10/C100, ViT-Tiny/C10,
ViT-Tiny/C100.  LionVote at cadence~8 converges fastest on
ViT-Tiny/CIFAR-100, reaching the $68\%$ threshold ${\sim}80$ epochs
before Lion.  On WRN configurations (top row), SGD with cosine
annealing achieves the highest final accuracy despite slower initial
convergence.}
\label{fig:convergence}
\end{figure}

LionVote at cadence~8 reaches the $68\%$ threshold ${\sim}80$ epochs
before Lion on ViT-Tiny/CIFAR-100 (Figure~\ref{fig:convergence};
full convergence thresholds in Appendix, Table~\ref{tab:convergence}).
Validation loss curves and generalisation gap analysis are in
Appendix~\ref{app:exp:Vloss_curves} and~\ref{app:exp:gen_gap}.

\subsection{Ablation Study}
\label{sec:experiments:ablation}

\begin{table}[t]
\centering
\caption{Ablation on WRN-28-10/CIFAR-10 and ViT-Tiny/CIFAR-100
(\%, mean $\pm$ std, 8 seeds).  Lion (no voting) included for reference.
Best numerical result per column in \textbf{bold};
the asymmetric update rule (c8) is the default configuration.
Cadence and max-level sensitivity in Appendix~\ref{app:exp:cadence}.}
\label{tab:ablation}
\small
\begin{tabular}{l l c c}
\toprule
Variant & Description & WRN/C10 & ViT/C100 \\
\midrule
Lion        & no voting & 93.71{\scriptsize$\pm$0.40} & 68.95{\scriptsize$\pm$0.41} \\
\midrule
LionVote (c4) & full system      & 93.14{\scriptsize$\pm$0.38} & 68.90{\scriptsize$\pm$0.48} \\
LionVote (c8) & full, cadence 8  & 93.30{\scriptsize$\pm$0.82} & 69.71{\scriptsize$\pm$0.60} \\
v2only      & Vote 2 only        & 93.46{\scriptsize$\pm$0.68} & 68.93{\scriptsize$\pm$0.41} \\
v1only      & Vote 1 only        & 92.22{\scriptsize$\pm$0.64} & 68.57{\scriptsize$\pm$0.49} \\
notie       & no tiebreaker      & 92.55{\scriptsize$\pm$0.34} & 68.64{\scriptsize$\pm$0.52} \\
symmetric   & symmetric update   & 92.88{\scriptsize$\pm$0.96} & \textbf{70.54}{\scriptsize$\pm$0.94} \\
\midrule
LionVote+none (c4)  & no schedule & 93.02{\scriptsize$\pm$0.52} & 67.89{\scriptsize$\pm$0.70} \\
LionVote+none (c8)  & no schedule & \textbf{93.48}{\scriptsize$\pm$0.39} & 68.89{\scriptsize$\pm$0.68} \\
\bottomrule
\end{tabular}
\end{table}

\emph{Vote~2 monitors momentum health and finds it stable.}
The v2only variant matches Lion on both configurations ($93.46\%$
WRN/C10, $68.93\%$ ViT/C100); the momentum-to-gradient ratio rarely
exits the $[1/e, e]$ dead zone (Appendix~\ref{app:exp:v2only}).

\emph{The dominant component is configuration-dependent.}  On WRN/C10,
the v1only variant (which disables both Vote~2 and the tiebreaker)
degrades to $92.22\%$, a $0.92$~pp gap from the full system
($93.14\%$).  Of this gap, $0.59$~pp is attributable to the tiebreaker
(Welch $p = 0.006$); the remaining $0.33$~pp attributed to Vote~2 is
not significant ($p = 0.23$), consistent with its near-zero firing
rate.  On ViT/C100, the pattern reverses: v1only achieves $68.57\%
\pm 0.49$, \emph{below} Lion ($-0.38$~pp, $p = 0.12$); notie
achieves $68.64\% \pm 0.52$ ($-0.31$~pp vs.\ Lion, $p = 0.22$),
while the full system at cadence~8 reaches $69.71\%$ ($+0.76$~pp,
$p = 0.017$).  The entire gain on ViT/C100 is attributable to the
tiebreaker: LionVote~c8 exceeds notie by $+1.07$~pp ($p = 0.003$).
Vote~1 alone and Vote~1+2 without the tiebreaker both fall below Lion
(neither difference significant).  LionVote thus combines per-layer
gradient diagnostics with a global validation loss signal: Vote~1
provides differentiated per-layer votes (negative $56\%$ of the time
for attention vs.\ $30\%$ for normalisation), while the tiebreaker
applies the same direction to all undecided layers whenever Vote~1
abstains (approximately $43\%$ of voting epochs for attention/MLP
and $69\%$ for normalisation, estimated from standard Lion gradient
statistics; Appendix~\ref{app:exp:raw_diagnostics}).
The per-layer mechanism creates divergent compound level trajectories
(Figure~\ref{fig:compound}), but the global tiebreaker signal is
necessary to convert this differentiation into an accuracy gain on
ViT/C100.

The voting mechanism works with either update rule.  On the two
configurations tested with symmetric updates, the difference from
asymmetric is not significant: $-0.26$~pp on WRN/C10 ($p = 0.49$)
and $+0.83$~pp on ViT/C100 ($p = 0.060$).  The symmetric variant
achieves the highest single-configuration accuracy in the study:
$70.54\% \pm 0.94$ on ViT/C100, significantly above Lion
($+1.59$~pp, $p = 0.002$), though with higher cross-seed variance
than the asymmetric default (std $0.94$ vs.\ $0.60$).  The choice
of update rule is a design decision, not a derivation
(\S\ref{sec:analysis:limitations}).

\section{Analysis}
\label{sec:analysis}
\subsection{Layer-Type Differentiation}
\label{sec:analysis:layers}

\begin{table}[t]
\centering
\caption{Mean compound level by layer type, ViT-Tiny/CIFAR-100,
LionVote + cosine warmup, cadence~4 (8 seeds, 12 blocks).  The
effective LR multiplier at epoch~300 is $\exp(s \cdot
0.9/2)$; SD$_{\mathrm{seed}}$ is the standard deviation of the
multiplier across 8 seeds (embed+patch and head+final\_norm each
contain too few tensors per seed for meaningful per-type SDs).}
\label{tab:compound_vit}
\small
\begin{tabular}{l r r r r r r r r}
\toprule
Layer type & ep4 & ep64 & ep124 & ep184 & ep244 & ep300 & Mult.\ & SD$_{\mathrm{seed}}$ \\
\midrule
all\_attn       & $-0.50$ & $-1.36$ & $-2.12$ & $-2.55$ & $-2.49$ & $-2.15$ & $0.38\times$ & $\pm 0.59$ \\
all\_mlp        & $-0.52$ & $-0.99$ & $-1.73$ & $-2.28$ & $-2.34$ & $-2.29$ & $0.36\times$ & $\pm 0.69$ \\
all\_norm       & $-0.57$ & $-0.80$ & $-0.86$ & $-1.06$ & $-1.31$ & $-1.54$ & $0.50\times$ & $\pm 0.26$ \\
embed+patch     & $-0.41$ & $-1.47$ & $-1.88$ & $-2.09$ & $-2.16$ & $-2.03$ & $0.40\times$ & --- \\
head+final\_norm & $-0.25$ & $-1.06$ & $-1.06$ & $-1.38$ & $-0.94$ & $-1.56$ & $0.50\times$ & --- \\
\bottomrule
\end{tabular}
\end{table}

By epoch~300 (Table~\ref{tab:compound_vit}), Lion's effective scale is
$2.6{\times}$ too high for attention parameters, $2.8{\times}$ too high
for MLP, and $2.0{\times}$ too high for normalisation.  Normalisation
layers receive $32\%$ higher effective scale than attention layers
($0.50 / 0.38 = 1.32$); this ratio is consistent at cadence~8
($1.33$; Appendix~\ref{app:exp:compound}).  Because the compound
multiplier $\alpha_i$ scales both the sign update and decoupled weight
decay, the absolute reductions measure joint LR+WD miscalibration,
and a global rate reduction cannot reproduce the per-layer
differentiation regardless of the LR/WD decomposition.

\begin{figure}[t]
\centering
\includegraphics[width=0.9\textwidth]{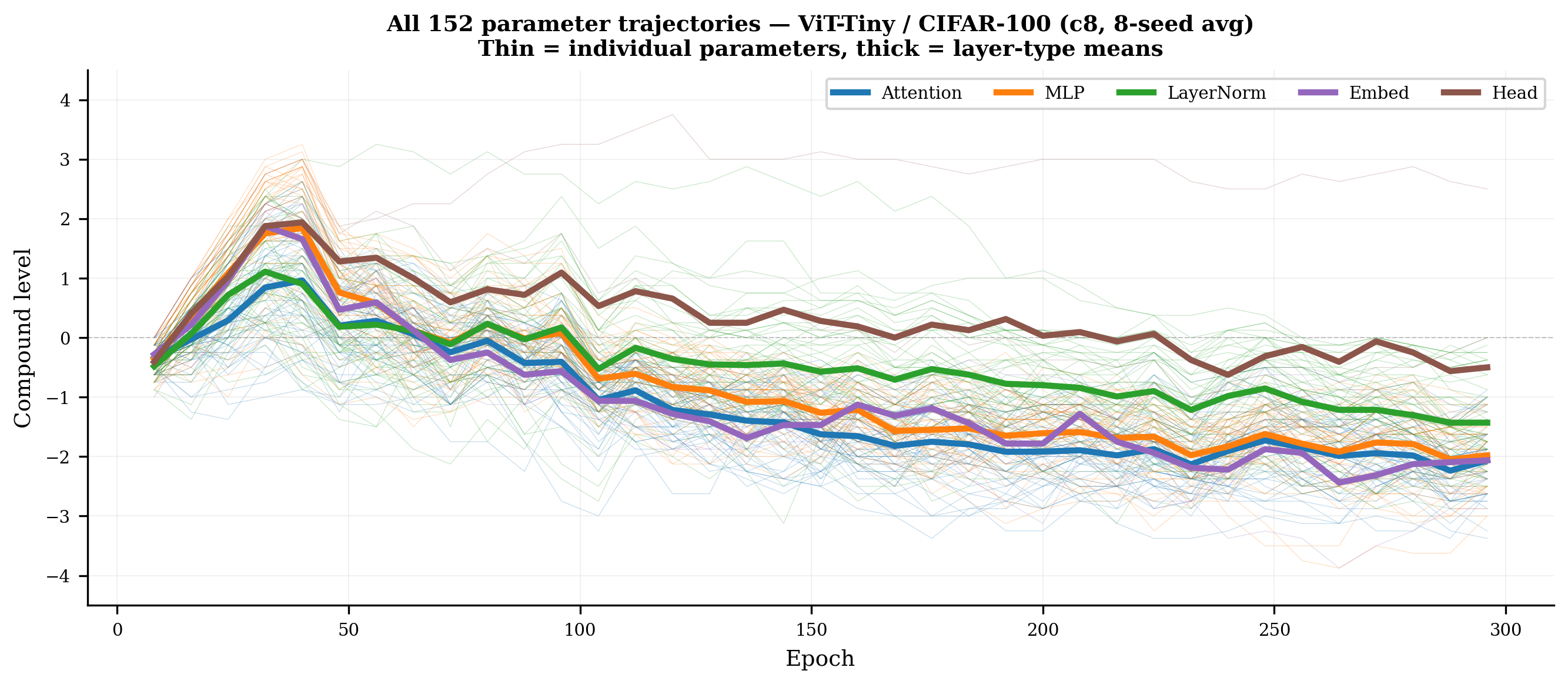}
\caption{Compound level trajectories for all 152 parameters of
ViT-Tiny on CIFAR-100 (cadence~8, 8-seed average).  Thin lines show
individual parameters; thick lines show layer-type means.  All layer
types drift negative over training, but normalisation layers (green)
and the classification head (brown) are penalised less than attention
(blue) and MLP (orange) parameters, reflecting a ${\sim}33\%$
difference in effective scale by epoch~296.  Within-type spread
(${\sim}4$ levels) substantially exceeds between-type mean differences
($0.6$ levels).}
\label{fig:compound}
\end{figure}

All levels are negative (Figure~\ref{fig:compound}), but the
$32\%$ effective-scale disparity between normalisation and attention
layers, consistent across cadence~4 (ratio $1.32$) and cadence~8
(ratio $1.33$; Appendix~\ref{app:exp:compound}), confirms that a
uniform rate reduction cannot substitute for per-layer
differentiation.

The compound trajectories quantify the disparity that
Zhao et~al.~\citep{zhao2025deconstructing} identify qualitatively
on autoregressive language models
(\S\ref{sec:related:perlayer}): the layers they flag as needing
distinct treatment (normalisation and the classification head) are
the same layers that the voting system assigns the highest effective
rates.

The differential treatment has a structural basis.  LayerNorm
parameters are elementwise scalars whose gradient distributions differ
qualitatively from the large projection matrices in attention and MLP
layers.  The voting system detects this through gradient alignment
statistics (Vote~1) without requiring architectural knowledge;
standard Lion training without voting produces the same layer-type
ordering (Appendix~\ref{app:exp:raw_diagnostics}).
Cross-optimizer diagnostics confirm the pattern: the
attn~$<$~norm ordering holds under both Lion and AdamW, but Lion's
$\operatorname{sign}(\cdot)$ amplifies the gap ($1.88{\times}$ larger
norm--attn cosine difference; $0.109$ vs.\ $0.058$).  AdamW's second
moment smooths layer-type differences; Lion's sign operation preserves
raw gradient structure, making per-layer adaptation more consequential
(Appendix~\ref{app:exp:raw_diagnostics}).

Proposition~\ref{prop:cosh} predicts the observed operating regime:
the convergence bound degradation at compound level~$s$ grows as
$\cosh(s \cdot \beta_1/2)$, reaching $+43\%$ at level~$-2$ where
most layer types settle, and $+106\%$ at level~$-3$ which appears
only transiently (Appendix~\ref{app:exponent}).

\paragraph{What transfers beyond Lion.}
Running LionVote's diagnostics against a paired AdamW trajectory
on the same workload
(Appendix~\ref{app:exp:raw_diagnostics}) separates the
optimizer-agnostic and Lion-specific components of the mechanism.
Vote~1's input signal (cosine of epoch-mean gradients) and firing
rates are near-invariant between Lion and AdamW
($V_1 = -1$~rate $37.6\%$ vs.\ $35.4\%$); its upper threshold's
derivation uses Lion's sign structure, but the $2/3$ supermajority
convention is optimizer-agnostic.  Vote~2's thresholds, derived
from Lion's $\beta_2$-EMA time constant, misfire under AdamW's
different momentum mechanics ($V_2 = -1$ fires on $99.98\%$ of
AdamW observations vs.\ $2.43\%$ under Lion); these thresholds are
Lion-specific and would need re-derivation for transfer.  The
tiebreaker, compound-level state machine, and cadence/max-level
hyperparameters are optimizer-agnostic.  A port to another
momentum-based optimizer would preserve the voting architecture,
recalibrate Vote~2 against that optimizer's effective momentum
time constant, and re-examine the LR exponent.  A full breakdown
is in Table~\ref{tab:portability}.

\subsection{Architecture Dependence}
\label{sec:analysis:arch}

On WRN-28-10/CIFAR-10, the three residual groups (layer1, layer2,
layer3) settle to compound levels of $-1.68$, $-1.62$, and $-1.84$ at
epoch~200, a spread of $0.22$ levels between the most and least
reduced groups (Appendix~\ref{app:exp:compound}).  On ViT-Tiny, the
spread between attention ($-2.15$) and normalisation ($-1.54$) is
$0.61$ levels, $2.8{\times}$ larger.

The difference reflects architectural heterogeneity.  WRN groups are
structurally uniform: each is a sequence of conv+BN blocks with the
same connectivity pattern at different spatial resolutions.  ViT
blocks contain three qualitatively different layer types, each with
distinct gradient dynamics.  The voting system finds more to
differentiate in the ViT.

However, heterogeneity is necessary but not sufficient.  At cadence~8,
the ViT attn-norm spread is $0.64$ levels on CIFAR-10 and $0.63$
levels on CIFAR-100 (nearly identical), yet LionVote significantly
helps on CIFAR-100 ($+0.76$~pp, $p = 0.012$) and significantly hurts
on CIFAR-10 ($-0.88$~pp, $p < 0.001$), where all cadences degrade
performance.  What distinguishes CIFAR-100 from CIFAR-10 is not
architectural spread but a task-level factor; two tasks do not
suffice to identify it.  Additional architectures and tasks are
needed to map the boundary of per-layer adaptation value.

LionVote without a schedule matches its scheduled variant on WRN/C10
but falls short on ViT/C100
(Appendix~\ref{app:exp:schedule_replacement}).

\subsection{Limitations}
\label{sec:analysis:limitations}

Experiments are limited to CIFAR-10 and CIFAR-100 with WideResNet and
ViT-Tiny.  The scope is deliberate: every configuration is run with
8~random seeds to eliminate variance-driven conclusions, a level of
replication that is incompatible with large-scale experiments under
limited compute.  Whether the layer-type differentiation observed
here (attention/MLP penalised relative to normalisation) holds at
larger scale and on different architectures is an open question.
\paragraph{Testing hierarchy.}
Cadence is selected by ablation from $c \in \{4, 6, 8\}$, making the
LionVote~c8 vs.\ Lion comparison partially post-hoc.  Four control
experiments on ViT-Tiny/CIFAR-100 (8 seeds each;
Appendix~\ref{app:exp:controls}) test the per-layer adaptation claim.
Lion at a uniformly lower LR ($4{\times}10^{-5}$, matching LionVote's
average effective rate) is indistinguishable from standard Lion
($p = 0.30$) and significantly below LionVote~c8 ($p = 0.003$,
survives Bonferroni correction).  Static per-layer-type multipliers
(epoch-300 values from Table~\ref{tab:compound_vit}) perform worse than
standard Lion ($p = 0.024$), confirming that the non-monotonic trajectory
cannot be replicated by static scheduling.  A weight decay sweep
($\mathrm{wd} \in \{0.25, 0.5, 0.75\}$) does not close the gap
($p \geq 0.075$).  The $c{=}8$ vs.\ Lion comparison ($p = 0.017$) is
marginal after Holm--Bonferroni correction; the primary evidence rests on
the control comparison.  The study is underpowered below Cohen's
$d \approx 0.8$.

\paragraph{Practitioner guidance.}
LionVote at cadence~8 is most likely to help on architectures with
qualitatively distinct layer types (ViT-style).  On uniform CNNs or
tasks where Lion is near its ceiling, it provides no consistent advantage.

\paragraph{Additional limitations.}
The tiebreaker uses validation loss during training (${\sim}37$ decisions
at cadence~8); future work should confirm on a held-out split.  Vote~2
rarely fires at Lion's recommended rate
(\S\ref{sec:experiments:ablation}).  LionVote exhibits higher cross-seed
variance than Lion (e.g.\ std $0.82$ vs.\ $0.40$ on WRN-28-10/C10).
The symmetric variant ($70.54\%$ on ViT/C100, $p = 0.002$) was tested on
only 2 of 4 configurations.

\section{Conclusion}
\label{sec:conclusion}
LionVote is a per-layer learning rate mechanism for Lion whose
voting thresholds are derived from mathematical structure rather than
tuned.  Three findings are useful beyond the specific method.
 
First, the threshold derivation methodology
(Appendix~\ref{app:vote1}--\ref{app:cadence})---Sheppard's
sign-agreement formula, the EMA time constant, and the descent lemma
for sign-based updates---provides reusable building blocks for per-layer
mechanisms in other optimizers.
 
Second, Lion's effective scale is systematically miscalibrated across
layer types on ViT-Tiny, with a $32\%$ disparity between
normalisation and attention layers (\S\ref{sec:analysis:layers}).
Because the compound multiplier scales both the sign update and
decoupled weight decay, this is a joint LR+WD finding.  A weight decay
sweep (WD${=}0.25$, WD${=}0.75$; \S\ref{sec:analysis:limitations})
confirms that neither WD adjustment closes the LionVote gap, ruling out
simple global WD miscalibration as the explanation and supporting the
per-layer adaptation claim.  This complements
\citet{zhao2025deconstructing} by quantifying the disparity under Lion
on vision architectures.  Future work on Lion should consider
layer-type-aware scaling even without the full voting mechanism.
The voting mechanism is not tied to the asymmetric update rule: a
symmetric variant achieves $70.54\%$ on ViT/C100 ($p = 0.002$ vs.\
Lion, tested on 2 of 4 configurations), the study's best single
result.

Third, per-layer adaptation provides more value as architectural
heterogeneity increases.  The compound level spread between layer
types is $2.8{\times}$ larger on ViT-Tiny than on WideResNet,
and the accuracy benefit tracks this spread.  This suggests that
architectures with diverse layer types are the natural setting for
per-layer methods.
 
Open questions include whether the voting mechanism
generalises to other sign-based or adaptive optimizers, whether
cadence can be adapted during training rather than fixed, and whether
training regimes exist---such as higher base learning rates or
different $\beta_2$ values---in which Vote~2 (momentum health)
becomes a primary driver of compound level movement.

\newpage
\bibliographystyle{plainnat}


\newpage
\appendix
\raggedbottom  
\section{Hyperparameters}
\label{app:hyperparams}

LionVote's voting thresholds are not tuned.  Vote~1's lower boundary is a geometric identity; its upper boundary connects a classical supermajority convention to a unique cosine threshold via the sign structure of Lion's update rule.  Vote~2's boundaries follow from the intrinsic timescale of the $\beta_2$ exponential moving average.  The tiebreaker threshold is the noise floor of cross-entropy loss estimation for standard validation set sizes; unlike Votes~1 and~2, it is determined to order of magnitude rather than uniquely.  The per-layer LR exponent $\beta_1/2$ couples to Lion's momentum-gradient balance through $\beta_1$ and to the curvature exponent of $L$-smooth objectives through the divisor~$2$.

\subsection{Vote 1: Gradient Direction Stability (Thresholds 0 and 0.5)}
\label{app:vote1}

Vote~1 computes
\[
  c \;=\; \cos(\bar{g}_{\mathrm{curr}},\,\bar{g}_{\mathrm{prev}}),
\]
the cosine similarity between the current and previous epoch-mean gradient snapshots, and returns $+1$
if $c > 0.5$, $-1$ if $c < 0$, and $0$ otherwise. The two thresholds
partition $[-1,1]$ into a \emph{reversal zone} ($c < 0$), a \emph{dead
band} ($c \in [0, 0.5]$), and a \emph{stability zone} ($c > 0.5$). We
justify each boundary in turn, then characterise the dead band.

\subsubsection*{Lower boundary: \texorpdfstring{$c = 0$}{c = 0}}

The cosine similarity satisfies $c < 0$ if and only if the angle between
$\bar{g}_{\mathrm{curr}}$ and $\bar{g}_{\mathrm{prev}}$ exceeds $\pi/2$:
the epoch-mean gradient has reversed direction across the voting window.
This is a complete, parameter-free geometric criterion.

\subsubsection*{Upper boundary: \texorpdfstring{$c = 0.5$}{c = 0.5}}

The lower boundary establishes $c = 0$ as the onset of positive alignment.
The upper threshold determines where alignment is strong enough to warrant
a positive vote. Lion applies $\operatorname{sign}(\cdot)$ coordinate-wise;
the operationally relevant quantity for a sign-based update is the fraction
of gradient coordinates that maintain their sign between voting epochs.
We show that $c = 0.5$ is the unique threshold corresponding to exactly a
two-thirds supermajority of coordinate sign pairs agreeing. This is a consequence
of the exact trigonometric identity $\arccos(1/2) = \pi/3$.

\begin{lemma}[Sheppard's sign-agreement formula]
\label{lem:sheppard}
For $(X, Y)$ jointly Gaussian with zero means, equal variances, and
correlation $\rho \in (-1,1)$,
\[
  P\!\bigl(\operatorname{sign}(X) = \operatorname{sign}(Y)\bigr)
  \;=\; 1 - \frac{\arccos(\rho)}{\pi}.
\]
\end{lemma}

\begin{proof}
\begin{align*}
  P(XY > 0)
  &= 2\,P(X > 0,\,Y > 0)
   = 2\!\left[\frac{1}{4} + \frac{\arcsin(\rho)}{2\pi}\right]
   = \frac{1}{2} + \frac{\arcsin(\rho)}{\pi}
   = 1 - \frac{\arccos(\rho)}{\pi},
\end{align*}
using $\arcsin(\rho) + \arccos(\rho) = \pi/2$.
\end{proof}

\begin{theorem}[Sign-agreement at the upper threshold]
\label{thm:supermajority}
Assume coordinate isotropy: for each coordinate $i$, the pair
$(\bar{g}_{\mathrm{curr},i},\,\bar{g}_{\mathrm{prev},i})$ is drawn from a
bivariate zero-mean Gaussian with equal variances and correlation $c$.
At $c = 0.5$, the probability that a coordinate pair agrees in sign is
exactly $2/3$. Moreover, $c = 0.5$ is the unique value in $(-1,1)$ for
which this probability equals $2/3$.
\end{theorem}

\begin{proof}
Since $\arccos(1/2) = \pi/3$ exactly, Lemma~\ref{lem:sheppard} gives
$P = 1 - (\pi/3)/\pi = 2/3$.
Uniqueness follows from the bijectivity of $\arccos$ on $[-1,1]$: the
equation $1 - \arccos(\rho)/\pi = 2/3$ has the unique solution $\rho = 1/2$.
\end{proof}

The threshold $c = 0.5$ is therefore the unique value in $(-1,1)$ at which
a two-thirds supermajority of gradient coordinate sign pairs are stable across
voting epochs.

\begin{remark}
The coordinate isotropy assumption treats each coordinate pair as drawn
from a common bivariate Gaussian whose correlation equals the global cosine
similarity $c$. Under gross violations of this assumption the sign-agreement
probability at $c = 0.5$ may differ from $2/3$, and
Theorem~\ref{thm:supermajority} should be read as an approximation in those
cases.
\end{remark}

\paragraph{Empirical validation.}
On ViT-Tiny/CIFAR-100 (seed~2, cadence~8), the per-coordinate
sign-agreement fraction was measured at each voting epoch for all
parameters whose global cosine similarity falls within a given bin.
Table~\ref{tab:sign_agreement} compares the Sheppard prediction
(Lemma~\ref{lem:sheppard}) against the observed fraction.

\begin{table}[H]
\centering
\small
\caption{Predicted vs.\ observed coordinate sign-agreement fraction
by cosine similarity bin, ViT-Tiny/CIFAR-100 (seed~2, cadence~8,
all voting epochs).  ``Predicted'' is
$1 - \arccos(c_{\mathrm{mid}})/\pi$ from
Lemma~\ref{lem:sheppard}.}
\label{tab:sign_agreement}
\begin{tabular}{c c c c}
\toprule
Cosine bin & Predicted & Observed & $n$ \\
\midrule
$[0.30, 0.35)$ & $0.605$ & $0.615$ & 85 \\
$[0.35, 0.40)$ & $0.622$ & $0.634$ & 91 \\
$[0.40, 0.45)$ & $0.640$ & $0.657$ & 57 \\
$[0.45, 0.50)$ & $0.658$ & $0.684$ & 37 \\
$[0.50, 0.55)$ & $0.676$ & $0.684$ & 22 \\
$[0.55, 0.60)$ & $0.695$ & $0.727$ & 17 \\
$[0.60, 0.65)$ & $0.715$ & $0.766$ &  9 \\
$[0.65, 0.70)$ & $0.736$ & $0.781$ &  7 \\
\bottomrule
\end{tabular}
\end{table}

At $c \in [0.45, 0.55]$ ($n = 59$ measurements), the mean observed
sign agreement is $0.684$, within $2.6\%$ of the predicted $2/3$.
The prediction tracks the data across the full range $c \in [0.3, 0.7]$,
with a consistent positive bias of $1$--$7\%$ (increasing with
cosine; upper bins have $n < 10$): real ViT-Tiny gradients
exhibit slightly more sign agreement than coordinate isotropy predicts.
The bias makes the $0.5$ threshold conservative---the true two-thirds
boundary lies slightly below $0.5$.  This is a single-seed measurement;
the direction of the bias is consistent across all eight cosine bins.
The $2/3$ supermajority convention is itself a design choice; Appendix~\ref{app:threshold_robustness} sweeps the resulting upper threshold across $[0.3, 0.7]$ on recorded Lion diagnostics and shows the vote composition does not sit near a phase transition in this range.

\subsubsection*{Dead band}

The two boundaries together define the dead band $c \in [0, 0.5]$. Within
it, the gradient direction has not reversed, but under the coordinate
isotropy model strictly fewer than two-thirds of coordinate sign pairs are
stable. Vote~1 abstains, treating this as an ambiguous regime where neither
direction is warranted.

\subsection{Vote 2: Momentum Health (Thresholds \texorpdfstring{$1/e$}{1/e} and \texorpdfstring{$e$}{e})}
\label{app:vote2}

Vote~2 computes
\[
  r \;=\; \frac{\|m\|}{\|\bar{g}_{\mathrm{curr}}\| + \epsilon},
\]
the ratio of the momentum EMA norm to the epoch-mean gradient norm, and
returns $-1$ if $r > e$, $+1$ if $r < 1/e$, and $0$ otherwise. Both
thresholds derive from the EMA's intrinsic timescale as a single symmetric pair.

\begin{lemma}[EMA time-constant identity]
\label{lem:ema_tau}
The EMA $m_t = \beta_2 m_{t-1} + (1-\beta_2)g_t$ has time constant
$\tau = -1/\ln\beta_2$. For any $\beta_2 \in (0,1)$,
\[
  \beta_2^{\,\tau} \;=\; e^{-1}.
\]
\end{lemma}
\begin{proof}
Consider the unit-step response: if $g_t$ switches from $0$ to $1$ at
$t=0$, unrolling the recurrence gives
\[
  m_t \;=\; 1 - \beta_2^{\,t}.
\]
In standard signal-processing usage, the \emph{time constant} $\tau$ of
such a first-order system is the time at which the step response reaches
$1 - 1/e \approx 63.2\,\%$ of its final value. Setting
$m_\tau = 1 - 1/e$:
\[
  1 - \beta_2^{\,\tau} \;=\; 1 - \frac{1}{e}
  \qquad\Longrightarrow\qquad
  \beta_2^{\,\tau} \;=\; e^{-1}.
\]
Taking logarithms of both sides,
\[
  \tau\ln\beta_2 \;=\; -1
  \qquad\Longrightarrow\qquad
  \tau \;=\; -\frac{1}{\ln\beta_2}.
\]
No free parameter enters: the identity $\beta_2^{\,\tau} = e^{-1}$
holds for every $\beta_2 \in (0,1)$.
\end{proof}

Under stationarity, the EMA converges in mean to the gradient mean
($\mathbb{E}[m_t] = \mathbb{E}[g_t]$), so $r$ concentrates near $1$. 
Deviations signal non-stationarity.

\begin{proposition}[Ratio dynamics after a gradient magnitude change]
\label{prop:ratio_dynamics}
Suppose the EMA has converged to $m_0 = A\hat{u}$ and the gradient
magnitude changes instantaneously from $A$ to $B$ with direction
preserved ($\hat{u}$ denotes the common unit direction). Then for $t \geq 0$,
\[
  r(t) \;=\; \frac{\|m_t\|}{B} \;=\; 1 + \beta_2^t\!\left(\frac{A}{B}-1\right).
\]
In particular, $r(0) = A/B$.
\end{proposition}

\begin{proof}
Direction is preserved, so $\|m_t\| = \beta_2^t A + (1-\beta_2^t)B$.
Dividing by $B$ gives the result.
\end{proof}

\begin{corollary}[The $(1/e,\,e)$ threshold pair]
\label{cor:threshold_pair}
The pair $(1/e,\,e)$ is the unique symmetric pair $(a,\,1/a)$ with
$a \in (0,1)$ satisfying $a = \beta_2^{\tau}$.  The interval $[1/e,\,e]$
has log half-width exactly 1, the natural unit of the EMA's
exponential decay $\beta_2^t = e^{-t/\tau}$, and is the only
log-symmetric interval $[e^{-\kappa},\,e^{\kappa}]$ centred at $1$
whose boundary coincides with $\beta_2^{\tau}$.
\end{corollary}

\begin{proof}
By Lemma~\ref{lem:ema_tau}, $\beta_2^{\tau} = 1/e$ for all
$\beta_2 \in (0,1)$. Setting $a = \beta_2^{\tau}$ gives $a = 1/e$ and
$1/a = e$ uniquely. The log half-width of the interval is $\ln e = 1$. 
Any alternative pair $(e^{-\kappa},\,e^{\kappa})$ with $\kappa \neq 1$
would satisfy $e^{-\kappa} \neq \beta_2^{\tau}$ and would require an
independent choice of $\kappa$.
\end{proof}

From Proposition~\ref{prop:ratio_dynamics}, $r(0) = A/B = e$ corresponds to
an instantaneous gradient magnitude drop by a factor of exactly $e$.
Corollary~\ref{cor:threshold_pair} identifies $e$ as the EMA's
fundamental unit of magnitude change, so $r = e$ marks a one-time-constant
excess of the EMA over the current gradient: the upper threshold fires
when the EMA retains more than one time-constant's worth of surplus signal.
Symmetrically, $r = 1/e$ marks a one-time-constant deficit, where the
gradient has surged by more than one $e$-fold relative to the EMA.
The dead zone $r \in [1/e,\,e]$ spans exactly two units in log scale.

\begin{remark}
Proposition~\ref{prop:ratio_dynamics} characterises $r$ at $t = 0$,
immediately after an instantaneous magnitude change. In practice, $r$ is
measured at epoch end: after $N$ training steps the EMA has already begun
adapting, and the observed ratio is closer to $1$ than $r(0) = A/B$.
Vote~2 therefore fires only when the gradient magnitude has been
persistently shifted over many steps, not after a transient change.
The thresholds $1/e$ and $e$ are principled anchors calibrated to the
EMA's natural unit, not exact characterisations of a specific
trajectory.
\end{remark}

\subsection{Tiebreaker: Validation Loss (Threshold \texorpdfstring{$\pm 1\%$}{±1\%})}
\label{app:tiebreaker}

Unlike Votes~1 and~2, whose thresholds derive from mathematical identities
independent of any experimental parameter, the tiebreaker threshold depends
on the validation set size~$n$.  We show that for $n \approx 10^4$, the
$\pm 1\%$ relative threshold matches the intrinsic noise floor of
cross-entropy loss estimation.

The tiebreaker computes the relative loss improvement
\[
  \delta \;=\;
    \frac{\bar{L}_{\mathrm{prev}} - \bar{L}_{\mathrm{curr}}}
         {|\bar{L}_{\mathrm{prev}}| + \epsilon},
\]
where $\bar{L} = \frac{1}{n}\sum_{i=1}^{n}\ell_i$ is the mean validation
loss with $\ell_i = -\log p_{\mathrm{correct},i}$, and returns $+1$ if
$\delta > 0.01$, $-1$ if $\delta < -0.01$, and $0$ otherwise.

By the central limit theorem, the standard error of $\bar{L}$ is
$\sigma_\ell/\sqrt{n}$.  Dividing by $\bar{L}$ gives the relative
standard error
\[
  \mathrm{RSE}
  \;=\; \frac{\sigma_\ell}{\bar{L}\,\sqrt{n}}
  \;=\; \frac{\mathrm{CV}_\ell}{\sqrt{n}},
\]
where $\mathrm{CV}_\ell = \sigma_\ell/\bar{L}$ is the coefficient of
variation of the per-sample loss.

The threshold therefore reduces to a question about
$\mathrm{CV}_\ell$.  For cross-entropy loss on a trained but imperfect
classifier, the loss $\ell_i = -\log p_{\mathrm{correct},i}$ is
supported on $[0,\infty)$ and right-skewed: most predictions have high
confidence and low loss, while a minority of hard examples have low
confidence and high loss.  Both $\bar{L}$ and $\sigma_\ell$ are
dominated by this same right tail, so their ratio
$\mathrm{CV}_\ell$ stays $\Theta(1)$ across models, datasets, and
training stages.

At $n \approx 10^4$ this gives
$\mathrm{RSE} \approx \mathrm{CV}_\ell / \sqrt{n} = \Theta(1\%)$.
The $1\%$ threshold is therefore the order of magnitude of the noise
floor at this validation set size.  It is not uniquely determined:
nearby values would be comparably justified.  It serves as a round
threshold that separates genuine loss movement from estimation noise.

\begin{remark}
The tiebreaker compares validation losses computed on the same fixed
sample set at consecutive voting epochs.  This paired structure has
strictly smaller standard error than the unpaired RSE derived above,
because the per-sample loss differences $d_i = \ell_i^{(t)} -
\ell_i^{(t-c)}$ have lower variance than the losses themselves when the
model changes only gradually.  The $1\%$ threshold is therefore
conservative: it acts only when the signal exceeds the noise floor of a
single unpaired measurement.
\end{remark}

\begin{remark}
For validation sets of substantially different size, the noise floor
shifts as $\mathrm{CV}_\ell/\sqrt{n}$.  At larger $n$ the noise floor
drops below $1\%$ and the fixed threshold becomes more conservative,
potentially missing genuine but modest loss changes.  At smaller $n$ the
noise floor exceeds $1\%$ and the fixed threshold becomes too aggressive,
potentially acting on sampling noise.  A size-adaptive threshold
$\mathrm{CV}_\ell/\sqrt{n}$ is a natural extension.  We use the fixed
value $0.01$ for simplicity.
\end{remark}

\subsection{LR Exponent: \texorpdfstring{$\beta_1/2$}{β₁/2}}
\label{app:exponent}
The per-layer learning rate formula
\[
  \alpha_i \;=\; \mathrm{lr} \cdot
    \exp\!\bigl(s_i \cdot \beta_1 / 2\bigr)
\]
involves two choices: $\beta_1$ as the numerator and $2$ as the divisor.
The numerator follows from Lion's momentum-gradient structure. The
divisor coincides with the curvature exponent universal to $L$-smooth
objectives.
\subsubsection*{Numerator: \texorpdfstring{$\beta_1$}{beta1}}

The boundary behaviours of $\beta_1$ pin down the numerator.
At $\beta_1 = 0$, $\exp(s_i \cdot 0/2) = 1$ for all $s_i$, so
$\alpha_i = \mathrm{lr}$ uniformly and per-layer modulation vanishes.
The modulation range
$\alpha_{+L}/\alpha_{-L} = \exp(L \cdot \beta_1)$ is strictly
increasing in $\beta_1$ (since $L > 0$), and reaches its maximum
$\exp(L)$ at $\beta_1 = 1$.

These boundaries reflect the optimizer's structure. When $\beta_1 = 0$,
Lion's update rule reduces to $\operatorname{sign}(g)$: the current
gradient alone, with no dependence on momentum history
\citep{chen2023symbolic}.  Every layer receives the same quality of
learning signal regardless of convergence state, so per-layer modulation
is unnecessary and the formula correctly produces none.  As
$\beta_1 \to 1$, the update approaches $\operatorname{sign}(m)$: pure
momentum.  Layers that have been learning in a stable direction
accumulate useful momentum, while layers in noisy or transitioning
regimes accumulate stale or contradictory momentum.

Making the exponent proportional to $\beta_1$
is the simplest form that (a)~vanishes when momentum is absent,
(b)~grows monotonically with the momentum fraction, and (c)~introduces
no free parameter beyond those already present in the optimizer.

\begin{remark}
This is a design argument, not a uniqueness theorem.  The boundary
behaviours constrain the numerator to be a function that equals zero at
$\beta_1 = 0$, is positive for $\beta_1 > 0$, and is increasing.  The
simplest such function is $\beta_1$ itself.  Other increasing functions
vanishing at zero, such as $\beta_1^2$ or $\sqrt{\beta_1}$, satisfy the
same boundary conditions but introduce curvature with no structural
motivation.
\end{remark}

\subsubsection*{Divisor: \texorpdfstring{$2$}{2}}

The divisor determines the per-level granularity of learning rate
adaptation.  We derive a quantitative criterion from the convergence
bound of sign-based methods on $L$-smooth objectives.

Under coordinate-wise $L$-smoothness, one step of a sign-based update
$x_{k+1} = x_k - \alpha\,\operatorname{sign}(v_k)$ satisfies
\[
  f(x_{k+1})
  \;\le\;
  f(x_k)
  - \alpha\,\langle \nabla\! f(x_k),\,\operatorname{sign}(v_k) \rangle
  + \frac{\alpha^2}{2}\,\|\tilde{L}\|_1,
\]
where $\|\tilde{L}\|_1 = \sum_i L_i$.  This is a consequence of the
coordinate-wise smoothness framework of
\citet{bernstein2018signsgd} (Theorem~1 and Lemma~E.1).  The key
simplification relative to standard gradient descent is that
$|\operatorname{sign}(\cdot)|^2 = 1$ renders the curvature penalty
$(\alpha^2/2)\|\tilde{L}\|_1$ independent of the gradient magnitude.

Telescoping over $K$ iterations and dividing by $\alpha K$ yields the
convergence bound
\[
  B(\alpha)
  \;=\; \frac{A}{\alpha} + C\alpha,
  \qquad
  A = \frac{\Delta f}{K},\quad
  C = \frac{\|\tilde{L}\|_1}{2},
\]
where the first term captures progress (decreasing in~$\alpha$) and the
second captures curvature cost (increasing in~$\alpha$).

\begin{proposition}[Log-LR perturbation cost]
\label{prop:cosh}
The optimal learning rate is $\alpha^* = \sqrt{A/C}$.  At a perturbed
rate $\alpha^* e^\varepsilon$:
\[
  \frac{B(\alpha^*\, e^\varepsilon)}{B(\alpha^*)}
  \;=\;
  \cosh(\varepsilon).
\]
The relative degradation depends only on the log-LR
perturbation~$\varepsilon$, not on~$\Delta f$, $\|\tilde{L}\|_1$, the
dimension, or the number of steps.
\end{proposition}

\begin{proof}
At the optimum, $A / \alpha^* = C\alpha^*$, so
$B(\alpha^*) = 2\sqrt{AC}$.  Then
\begin{align*}
  B(\alpha^* e^\varepsilon)
  &= \frac{A}{\alpha^* e^\varepsilon} + C\,\alpha^* e^\varepsilon
   = \sqrt{AC}\,(e^{-\varepsilon} + e^{\varepsilon})
   = 2\sqrt{AC}\,\cosh(\varepsilon).
\end{align*}
Dividing by $B(\alpha^*) = 2\sqrt{AC}$ gives $\cosh(\varepsilon)$.
\end{proof}

Each compound level shifts the log-LR by $\varepsilon = \beta_1/d$,
where $d$ is the divisor.  The per-level convergence cost is
$\cosh(\beta_1/d)$, and the worst case at $\pm L$ levels is
$\cosh(L \cdot \beta_1/d)$.  For $\beta_1 = 0.9$ and $L = 4$:

\begin{table}[H]
\centering
\small
\caption{Per-level and worst-case convergence bound degradation by divisor ($\beta_1 = 0.9$, $L = 4$).}
\label{tab:divisor_comparison}
\begin{tabular}{cccc}
\toprule
Divisor $d$
  & Per-level $\varepsilon$
  & $\cosh(\varepsilon)$
  & Worst case $\cosh(4\varepsilon)$ \\
\midrule
1 & 0.90 & 1.433\; (+43\%)   & 18.31 \\
2 & 0.45 & 1.103\; (+10\%)   & 3.107 \\
3 & 0.30 & 1.045\; (+4.5\%)  & 1.811 \\
\bottomrule
\end{tabular}
\end{table}

At $d = 1$, the worst-case degradation of $18\times$ renders the outer
levels unusable: a layer at level~$+4$ pays an order-of-magnitude
convergence penalty, defeating the purpose of a graded system.  At
$d = 3$, the per-level LR factor $\exp(0.30) \approx 1.35\times$
provides insufficient differentiation per vote, making the voting system
sluggish.

At $d = 2$, each compound level incurs ${\sim}10\%$ convergence bound
degradation, and the full $\pm 4$ range spans from ${\sim}0.17\times$ to
${\sim}6\times$ the base rate.  This permits meaningful per-layer
adaptation without excessive worst-case cost.

\begin{remark}
The $\cosh$ result (Proposition~\ref{prop:cosh}) applies per-layer: each
layer's parameters satisfy an independent descent lemma with their own
smoothness constants.  A layer at compound level~$s$ incurs degradation
$\cosh(s \cdot \beta_1/2)$ relative to the base rate.  The per-layer
nature of LionVote means these costs do not compound across layers.
\end{remark}

\subsubsection*{Consistency check: edge of stability}

At the edge of stability \citep{cohen2021eos}, training dynamics
self-organise so that the loss landscape sharpness hovers near the
optimizer's stability threshold.  When the learning rate increases, the
threshold drops: the loss landscape must flatten for training to remain
stable.  A one-level LR increase of $\exp(\beta_1/2) \approx 1.57\times$
is a moderate perturbation; the network has multiple epochs before the
next vote to re-equilibrate.  Using $\beta_1$ directly as the exponent
($d = 1$) gives $\exp(0.9) \approx 2.46\times$ per level, a
substantially larger perturbation that demands more aggressive landscape
rearrangement within the same time window.

\begin{remark}
The edge of stability was established empirically for full-batch gradient
descent, where the spectral norm of the Hessian converges to
approximately $2/\eta$ \citep{cohen2021eos}.  For sign-based methods
with momentum, the stability mechanism operates through a different
sharpness measure and the specific threshold differs.  The qualitative
principle that increasing the learning rate forces the loss landscape to
flatten is general across optimizers and is the basis of this
consistency check.
\end{remark}
\begin{remark}
No published result proves $\beta_1/2$ is uniquely optimal.  Any
divisor in $[1.5, 3]$ would be defensible.  The value $2$ is the
simplest integer in this range and balances per-level cost against
modulation range (Table~\ref{tab:divisor_comparison}).
\end{remark}


\subsection{Maximum Compound Level}
\label{app:maxlevel}

The compound level is bounded by $\pm L$, where $L = \texttt{max\_level}$.
The choice of $L$ involves a trade-off between expressivity (the range of
per-layer learning rates) and worst-case convergence cost.

At compound level $s$ with divisor $d = 2$, the learning rate multiplier
is $\exp(s \cdot \beta_1/2)$ and the convergence bound degradation is
$\cosh(s \cdot \beta_1/2)$ (Proposition~\ref{prop:cosh}).
For $\beta_1 = 0.9$:

\begin{table}[H]
\centering
\small
\caption{LR range and worst-case convergence bound degradation by maximum compound level ($\beta_1 = 0.9$, $d = 2$).}
\label{tab:maxlevel_comparison}
\begin{tabular}{cccc}
\toprule
$L$
  & LR range $\exp(L \cdot \beta_1)$
  & Max multiplier $\exp(L \cdot \beta_1/2)$
  & Worst case $\cosh(L \cdot \beta_1/2)$ \\
\midrule
2 & ${\times}6.0$   & ${\times}2.46$  & 1.433\; (+43\%) \\
3 & ${\times}14.9$  & ${\times}3.86$  & 2.058\; (${\times}2.1$) \\
4 & ${\times}36.6$  & ${\times}6.05$  & 3.107\; (${\times}3.1$) \\
5 & ${\times}90.0$  & ${\times}9.49$  & 4.797\; (${\times}4.8$) \\
6 & ${\times}221$\phantom{.0}   & ${\times}14.9$\phantom{0}  & 7.473\; (${\times}7.5$) \\
\bottomrule
\end{tabular}
\end{table}

Two constraints bound $L$ from above and below.

\emph{Lower bound (expressivity).}  Compound levels are earned through
consecutive same-sign votes.  At $L = 2$, the maximum multiplier is
${\times}2.46$, comparable to a single step of a typical schedule
drop, and the total LR range is ${\times}6.0$.  A system with this
range has limited ability to differentiate layers with substantially
different convergence needs.  At $L = 3$, the range grows to
${\times}14.9$ and the maximum multiplier to ${\times}3.86$: enough to
express meaningful differentiation, but still less than one and a half
orders of magnitude.

\emph{Upper bound (worst-case cost).}  The worst-case convergence bound
degradation $\cosh(L \cdot \beta_1/2)$ grows with $L$.  In the table
above, the growth accelerates: from $L = 2$ to $L = 3$ the worst case
increases by $44\%$ ($1.43 \to 2.06$); from $L = 3$ to $L = 4$ by
$51\%$ ($2.06 \to 3.11$); from $L = 4$ to $L = 5$ by $54\%$ ($3.11
\to 4.80$).  This reflects the transition of $\cosh$ from its quadratic
regime ($\cosh(x) \approx 1 + x^2/2$) to its exponential regime
($\cosh(x) \approx e^x/2$).  At $L = 4$, the argument to $\cosh$ is
$1.8$ and the quadratic approximation $1 + 1.8^2/2 = 2.62$
underestimates the true value $3.107$ by $16\%$.  At $L = 5$, the
argument is $2.25$ and the quadratic approximation gives $3.53$ versus
the true $4.797$, a $26\%$ underestimate.  The quadratic regime, where
$\cosh$ grows gently, is ending.

$L = 4$ is the operating point where the LR range first exceeds one and
a half orders of magnitude (${\times}36.6$) while the worst-case
degradation remains below ${\times}3.2$, moderate enough that a layer
at the boundary is not catastrophically penalised.

\begin{remark}
$L = 3$ ($\cosh(1.35) \approx 2.06$, range ${\times}14.9$) is also
defensible.  $L = 4$ is preferred because it provides $2.5{\times}$
the LR range at only $51\%$ additional worst-case cost, and because
finer granularity allows the system to make smaller distinctions
between layers at different stages of convergence.
\end{remark}

\subsection{Voting Cadence}
\label{app:cadence}

The voting cadence $c$ is the number of epochs between consecutive
voting events.  Unlike the preceding hyperparameters, cadence cannot
be derived from a single mathematical identity: it depends on the
interaction between the voting system, the compound level update rule,
and the training dynamics.  We characterise three structural
constraints that bound the useful range of $c$.

\subsubsection*{Lower bound: EMA convergence}

Vote~2 uses the momentum EMA $m_t = \beta_2 m_{t-1} + (1 - \beta_2)
g_t$.  After a learning rate change alters the gradient distribution, the
EMA requires $\tau = -1/\!\ln\beta_2$ steps to track the new mean
(Lemma~\ref{lem:ema_tau}).  Until the EMA has converged, the
momentum-to-gradient ratio $r$ (Vote~2's input) reflects the old regime
rather than the new one.  The voting window must be long enough that the
EMA statistics used for voting have equilibrated after any level change
from the previous vote.

\subsubsection*{Lower bound: recovery time}

Under the asymmetric update rule, an opposing vote resets the compound
level to zero.  After a reset, the system requires up to $L$
consecutive same-sign votes, i.e.\ $L \cdot c$ epochs, to return to
the boundary.  If $c$ is too small, the system votes before the
gradient statistics have had time to reflect the learning rate change
caused by the previous vote.  The votes then respond to stale
information, producing oscillatory level dynamics.

\subsubsection*{Upper bound: adaptivity}

The system has at most $\lfloor E/c \rfloor$ voting opportunities over
$E$ training epochs, minus one for the cold-start reference.  The
number of voting opportunities must comfortably exceed $2L$ to allow
both upward and downward traversal of the full $[-L, +L]$ range over
the course of training.  At $c = 16$ over $E = 200$ epochs, only $12$
votes occur; with $L = 4$, reaching the boundary and returning to zero
would consume $8$ of them, leaving little room for adaptation to
changing training dynamics.

\subsubsection*{Selection within the range}

The constraints above place the useful range at roughly $c \in [4, 10]$
for typical training lengths ($E \geq 200$ epochs in this work).  Within this
range, smaller cadence gives more frequent adaptation at the cost of
noisier votes and less recovery time. Larger cadence gives cleaner
votes at the cost of fewer total voting opportunities.

\begin{remark}
Cadence is the one hyperparameter in LionVote that is selected by
ablation rather than derived from a mathematical identity.  It is
bounded from below by EMA convergence time and recovery dynamics, and
from above by the need for sufficient voting opportunities.  Within the
feasible range, the choice is empirical: $c = 4$ is a conservative
default and $c = 8$ is preferred when total training epochs permit.
\end{remark}

\subsection{Configuring LionVote}
\label{app:practitioner_recipe}

This subsection consolidates the tunable and non-tunable knobs, their
defaults, and guidance for non-default choices.

\begin{table}[H]
\centering
\small
\caption{LionVote configuration summary.  ``Default'' columns give the
value used in all experiments in this paper.  ``Derivation'' cites the
appendix section that pins the default.}
\label{tab:config_summary}
\begin{tabular}{l c l p{0.40\textwidth}}
\toprule
Knob & Default & Derivation & Practitioner action \\
\midrule
Lion LR & $10^{-4}$ & -- & Use Lion's recommended rate for the workload. \\
Lion $\beta_1$ & $0.9$ & -- & Use Lion default (momentum interp). \\
Lion $\beta_2$ & $0.99$ & -- & Use Lion default (EMA buffer). \\
Weight decay & $0.5$ & -- & Use Lion default for ViT workloads. \\
Cadence $c$ & $8$ & App.~\ref{app:cadence} & If training $< 200$ epochs, use $c = 4$. \\
Max level $L$ & $4$ & App.~\ref{app:maxlevel} & Saturation rate at $L = 4$ is $11.7\%$
   (Table~\ref{tab:compound_distribution}); lowering to $L = 3$ truncates
   ${\sim}12\%$ of decisions. \\
Vote 1 lower & $0$ & App.~\ref{app:vote1} & Do not change. \\
Vote 1 upper & $0.5$ & App.~\ref{app:vote1} & Do not change;
   see Table~\ref{tab:v1_threshold_sensitivity} for perturbation impact on vote composition. \\
Vote 2 factor & $e$ & App.~\ref{app:vote2} & Do not change;
   Vote 2 is near-dormant at this value and that is by design. \\
Tiebreaker & $0.01$ & App.~\ref{app:tiebreaker} & Rescale as
   $\mathrm{CV}_\ell/\sqrt{n}$ for validation set sizes far from the default. \\
LR exponent divisor & $2$ & App.~\ref{app:exponent} & Do not change; defensible on worst-case
   (Table~\ref{tab:divisor_comparison}) and realised-mean
   (Table~\ref{tab:divisor_empirical}) grounds. \\
Asymmetric update rule & enabled & \S\ref{sec:experiments:ablation} &
   Consider the symmetric variant for ViT-style architectures. \\
\bottomrule
\end{tabular}
\end{table}

\paragraph{Per-iteration and memory cost.}
See Appendix~\ref{app:exp:computational_cost}: one additional in-place
accumulation per parameter per batch, $2$~tensors plus $2$~scalars of
extra state per parameter ($3{\times}$ Lion's optimizer-state memory),
and $152$~dot products and norms once every $c$ epochs for ViT-Tiny.
For billion-parameter models the $3{\times}$ state increase should be
weighed against Lion's optimizer-state footprint (${\sim}2{\times}$ the
model parameters).

\paragraph{When to expect benefit.}
LionVote at cadence~$8$ is most likely to help on architectures with
qualitatively distinct layer types (ViT-style).
On uniform CNNs or tasks where Lion is near its ceiling, it provides no
consistent advantage; see practitioner guidance in
\S\ref{sec:analysis:limitations}.  Per-layer adaptation value
increases with architectural heterogeneity
(\S\ref{sec:analysis:arch}).

\paragraph{Diagnostics-only deployment.}
Running LionVote with all votes disabled (and compound levels decaying
to zero) is equivalent to plain Lion.  This makes diagnostics
collection non-invasive: one can deploy LionVote with the voting
thresholds set to $\pm\infty$ (or using a flag) purely to collect
the per-layer gradient-alignment and momentum-ratio statistics, then
decide whether to enable voting.



\section{Experimental Details}
\label{app:exp:hparams}

\subsection{Weight Decay by Configuration}

Per-config, per-optimizer weight decay values.  Lion and LionVote use
identical weight decay; all ablation variants inherit LionVote's values.

\medskip
\begin{center}
\begin{tabular}{lccc}
\toprule
Configuration & AdamW & SGD & Lion / LionVote \\
\midrule
WRN-28-10 / CIFAR-10  & $5 \times 10^{-4}$ & $5 \times 10^{-4}$ & $5 \times 10^{-3}$ \\
WRN-40-10 / CIFAR-100 & $5 \times 10^{-4}$ & $5 \times 10^{-4}$ & $5 \times 10^{-3}$ \\
ViT-Tiny / CIFAR-10   & $5 \times 10^{-2}$ & $5 \times 10^{-4}$ & $5 \times 10^{-1}$ \\
ViT-Tiny / CIFAR-100  & $5 \times 10^{-2}$ & $5 \times 10^{-4}$ & $5 \times 10^{-1}$ \\
\bottomrule
\end{tabular}
\end{center}
\medskip

Lion's weight decay values follow the recommendation of
Chen et al.~\citep{chen2023symbolic} (Section~5): $3$--$10{\times}$
larger WD than AdamW, compensating for Lion's $3$--$10{\times}$
smaller learning rate.  AdamW uses PyTorch's built-in decoupled weight
decay; SGD uses L2 regularisation; Lion and LionVote apply decoupled
weight decay manually in the update step.

\paragraph{Augmentation note.}
CNN configurations on CIFAR-100 use the CIFAR-10 AutoAugment
policy.  PyTorch's \texttt{AutoAugmentPolicy.CIFAR10} is the only
CIFAR policy provided in \texttt{torchvision}; no separate CIFAR-100
policy exists.  Using the CIFAR-10 policy for both datasets is standard
practice in the CIFAR benchmark literature.

\subsection{Full Convergence Threshold Tables}
\label{app:exp:convergence}

Table~\ref{tab:convergence} provides selected convergence thresholds
for all configurations.

\begin{table}[H]
\centering
\caption{Epochs to reach accuracy thresholds (mean over seeds that
reached the threshold; only entries where all 8 seeds reached are
shown).  ``---'' = not reliably reached.}
\label{tab:convergence}
\small
\begin{tabular}{l l c c c c c}
\toprule
Configuration & Threshold & SGD & AdamW & Lion & LionVote (c4) & LionVote (c8) \\
\midrule
WRN-28-10 / C10   & $\to$91\% & 110 & 56 & 61 & \textbf{41} & 71 \\
                   & $\to$92\% & 123 & \textbf{76} & 82 & \textbf{76} & --- \\
                   & $\to$95\% & \textbf{158} & --- & --- & --- & --- \\
\midrule
WRN-40-10 / C100  & $\to$76\% & 143 & \textbf{119} & --- & --- & --- \\
                   & $\to$80\% & \textbf{169} & --- & --- & --- & --- \\
\midrule
ViT-Tiny / C10    & $\to$89\% & --- & 180 & 145 & 155 & \textbf{136} \\
                   & $\to$90\% & --- & 209 & \textbf{173} & 208 & \textbf{173} \\
\midrule
ViT-Tiny / C100   & $\to$67\% & --- & 214 & 199 & 146 & \textbf{133} \\
                   & $\to$68\% & --- & 250 & 239 & 193 & \textbf{160} \\
\bottomrule
\end{tabular}
\end{table}

Expanded per-configuration convergence tables follow (mean over 8
seeds; ``---'' = fewer than 8 seeds reached the threshold).

\paragraph{WRN-28-10 / CIFAR-10.}
\medskip
\begin{center}
\small
\begin{tabular}{l c c c c c c}
\toprule
Setup & $\to$88\% & $\to$89\% & $\to$90\% & $\to$91\% & $\to$92\% & $\to$93\% \\
\midrule
SGD+cos              &  64 &  84 &  95 & 110 & 123 & \textbf{136} \\
AdamW+cos            &  30 &  38 &  45 &  56 &  \textbf{76} & --- \\
Lion+cos             &  33 &  36 &  49 &  61 &  82 & --- \\
LionVote+cos         &  \textbf{27} &  \textbf{31} &  \textbf{34} &  \textbf{41} &  \textbf{76} & --- \\
LionVote+cos (c8)    &  33 &  40 &  45 &  71 & --- & --- \\
LionVote+cos (c6)    &  31 &  34 &  40 &  53 &  92 & --- \\
\bottomrule
\end{tabular}
\end{center}

\paragraph{WRN-40-10 / CIFAR-100.}
\medskip
\begin{center}
\small
\begin{tabular}{l c c c c c c}
\toprule
Setup & $\to$72\% & $\to$73\% & $\to$74\% & $\to$75\% & $\to$76\% & $\to$77\% \\
\midrule
SGD+cos              & 114 & 123 & 128 & 137 & 143 & \textbf{150} \\
AdamW+cos            &  58 &  66 &  79 &  98 & \textbf{119} & --- \\
Lion+cos             &  81 &  90 & 109 & 136 & --- & --- \\
LionVote+cos         &  \textbf{51} &  \textbf{61} &  \textbf{70} & --- & --- & --- \\
LionVote+cos (c6)    &  56 &  66 &  77 &  \textbf{93} & --- & --- \\
LionVote+cos (c8)    &  73 &  80 & 101 & --- & --- & --- \\
\bottomrule
\end{tabular}
\end{center}

\paragraph{ViT-Tiny / CIFAR-10.}
\medskip
\begin{center}
\small
\begin{tabular}{l c c c c c c}
\toprule
Setup & $\to$87\% & $\to$88\% & $\to$89\% & $\to$90\% & $\to$91\% & $\to$92\% \\
\midrule
Lion+cos+warm            & 105 & 121 & 145 & \textbf{173} & \textbf{211} & \textbf{257} \\
AdamW+cos+warm           & 137 & 161 & 180 & 209 & 249 & --- \\
LionVote+cos+warm (c8)   & 102 & \textbf{114} & \textbf{136} & \textbf{173} & 233 & --- \\
LionVote+cos+warm (c6)   &  \textbf{95} & 117 & 145 & 179 & --- & --- \\
LionVote+cos+warm        & 101 & 120 & 155 & 208 & --- & --- \\
SGD+cos+warm             & --- & --- & --- & --- & --- & --- \\
\bottomrule
\end{tabular}
\end{center}

\paragraph{ViT-Tiny / CIFAR-100.}
\medskip
\begin{center}
\small
\begin{tabular}{l c c c c c c}
\toprule
Setup & $\to$63\% & $\to$64\% & $\to$65\% & $\to$66\% & $\to$67\% & $\to$68\% \\
\midrule
LionVote+cos+warm (c8)  &  76 & 86 & 100 & \textbf{110} & \textbf{133} & \textbf{160} \\
LionVote+cos+warm (c6)  &  76 &  88 &  \textbf{98} & 114 & 135 & --- \\
Lion+cos+warm            & 85 & 104 & 130 & 158 & 199 & 239 \\
LionVote+cos+warm        &  \textbf{71} &  \textbf{80} &  \textbf{98} & 114 & 146 & 193 \\
AdamW+cos+warm           & 121 & 140 & 158 & 183 & 214 & 250 \\
SGD+cos+warm             & --- & --- & --- & --- & --- & --- \\
\bottomrule
\end{tabular}
\end{center}

\subsection{Cadence and Maximum Level Sensitivity}
\label{app:exp:cadence}

Best top-1 accuracy (\%) on WRN-28-10/CIFAR-10 with cosine schedule,
varying cadence $c$ and maximum level $L$ (8 seeds each).

\medskip
\begin{center}
\begin{tabular}{l c c c c}
\toprule
 & $c = 2$ & $c = 4$ & $c = 6$ & $c = 8$ \\
\midrule
$L = 2$ & 92.49{\scriptsize$\pm$0.79} & 92.95{\scriptsize$\pm$0.48} & 92.87{\scriptsize$\pm$0.62} & \textbf{93.47}{\scriptsize$\pm$0.74} \\
$L = 4$ & 91.22{\scriptsize$\pm$0.87} & 93.14{\scriptsize$\pm$0.38} & 93.03{\scriptsize$\pm$0.77} & 93.30{\scriptsize$\pm$0.82} \\
$L = 6$ & 90.98{\scriptsize$\pm$0.83} & 92.97{\scriptsize$\pm$0.63} & 92.72{\scriptsize$\pm$0.59} & 92.90{\scriptsize$\pm$1.00} \\
\bottomrule
\end{tabular}
\end{center}
\medskip

$c = 2$ is consistently worst, producing noisy votes before gradient
statistics equilibrate.  Within the $c \in [4, 8]$ range, accuracy is
relatively stable.  The best single entry ($93.47\%$ at $c{=}8$,
$L{=}2$) suggests that a small level range with clean votes can
outperform a large range with noisier ones.  Cadence~8 is preferred
when total training epochs ($\ge 200$) provide sufficient voting
opportunities.

\paragraph{ViT-Tiny/CIFAR-100 cadence sensitivity.}
Best top-1 accuracy (\%) on ViT-Tiny/CIFAR-100 with cosine warmup
schedule, $L = 4$ (8 seeds each).

\medskip
\begin{center}
\begin{tabular}{l c c c}
\toprule
 & $c = 4$ & $c = 6$ & $c = 8$ \\
\midrule
$L = 4$ & 68.90{\scriptsize$\pm$0.48} & 69.23{\scriptsize$\pm$0.67} & \textbf{69.71}{\scriptsize$\pm$0.60} \\
\bottomrule
\end{tabular}
\end{center}
\medskip

The cadence effect is larger on ViT-Tiny/CIFAR-100 than on
WRN-28-10/CIFAR-10: the gap between cadence~4 and cadence~8 is
$0.81$~pp (vs.\ $0.16$~pp on WRN at $L{=}4$).  This is consistent
with the analysis in \S\ref{sec:analysis:arch}: ViT's greater
layer-type heterogeneity makes Vote~1's gradient alignment estimates
more consequential.

\subsection{Threshold Robustness (Counterfactual Voting)}
\label{app:threshold_robustness}

The cadence~$c$ and maximum level~$L$ are swept empirically in
Appendix~\ref{app:exp:cadence}.  The \emph{inner} thresholds
(Vote~1 upper $0.5$, Vote~2 factor $e$, divisor $2$, tiebreaker $0.01$)
are derived from the conventions in
Appendix~\ref{app:vote1}--\ref{app:tiebreaker}; Vote~1's lower boundary
$0$ is the parameter-free geometric identity and is not tunable.  This
subsection reports post-hoc \emph{counterfactual} sensitivity: how the
vote composition would change under alternative thresholds, computed by
replaying LionVote's voting logic against the recorded
per-parameter gradient statistics from a plain-Lion training run
(Appendix~\ref{app:exp:raw_diagnostics}, seed~2, 36~measurement epochs,
152~parameters, $n = 5472$).  The replay characterises the threshold's
\emph{operating regime} but does not predict accuracy under retraining
at a different threshold; that would require fresh training runs, which
are outside the scope of this study.

\paragraph{Vote 1 upper threshold.}
Table~\ref{tab:v1_threshold_sensitivity} reports the vote
composition as the upper threshold sweeps across $[0.30, 0.70]$
with the lower threshold fixed at $0$.

\begin{table}[H]
\centering
\small
\caption{Vote~1 composition under alternative upper thresholds
(ViT-Tiny/CIFAR-100, Lion baseline, seed~2, $n = 5472$).}
\label{tab:v1_threshold_sensitivity}
\begin{tabular}{c c c c}
\toprule
Upper threshold & $+1$ rate & $0$ rate & $-1$ rate \\
\midrule
$0.30$ & $9.45\%$  & $53.00\%$ & $37.55\%$ \\
$0.40$ & $5.65\%$  & $56.80\%$ & $37.55\%$ \\
$0.45$ & $3.93\%$  & $58.52\%$ & $37.55\%$ \\
$\mathbf{0.50}$ (default) & $\mathbf{2.47\%}$ & $\mathbf{59.98\%}$ & $\mathbf{37.55\%}$ \\
$0.55$ & $1.30\%$  & $61.15\%$ & $37.55\%$ \\
$0.60$ & $0.80\%$  & $61.64\%$ & $37.55\%$ \\
$0.70$ & $0.35\%$  & $62.10\%$ & $37.55\%$ \\
\bottomrule
\end{tabular}
\end{table}

The $-1$~rate is fixed by the lower threshold $0$; perturbing the upper
threshold redistributes mass between $0$ and $+1$.  At the default
$0.50$, the $+1$ rate is already a minority behaviour on Lion
($2.47\%$); a $\pm 0.1$ perturbation changes the $+1$ rate by roughly a
factor of two but leaves the dominant $-1$ and abstain rates within
$2$~percentage points of baseline.  The method does not sit on a phase
transition in this regime.

\paragraph{Vote 2 factor.}
Table~\ref{tab:v2_threshold_sensitivity} sweeps the symmetric factor
$k$ (vote $+1$ if $r < 1/k$, $-1$ if $r > k$) from $1.25$ to $4.0$;
the default is $k = e \approx 2.718$.

\begin{table}[H]
\centering
\small
\caption{Vote~2 composition under alternative symmetric
factors~$k$ (same data as Table~\ref{tab:v1_threshold_sensitivity}).}
\label{tab:v2_threshold_sensitivity}
\begin{tabular}{c c c c}
\toprule
$k$ & $+1$ rate & $0$ rate & $-1$ rate \\
\midrule
$1.25$ & $0.00\%$ & $1.06\%$  & $98.94\%$ \\
$1.50$ & $0.00\%$ & $12.99\%$ & $87.01\%$ \\
$2.00$ & $0.00\%$ & $68.57\%$ & $31.43\%$ \\
$2.50$ & $0.00\%$ & $94.28\%$ & $5.72\%$  \\
$\mathbf{e \approx 2.718}$ (default) & $\mathbf{0.00\%}$ & $\mathbf{97.57\%}$ & $\mathbf{2.43\%}$ \\
$3.00$ & $0.00\%$ & $99.21\%$ & $0.79\%$  \\
$3.50$ & $0.00\%$ & $99.78\%$ & $0.22\%$  \\
$4.00$ & $0.00\%$ & $99.93\%$ & $0.07\%$  \\
\bottomrule
\end{tabular}
\end{table}

Three observations.  First, $+1$ rate is zero for every $k$ tested:
the momentum norm never falls below the gradient norm by a factor
$k$ on this workload.  The asymmetry is a property of Lion's training
dynamics on ViT-Tiny, not of the threshold.
Second, Vote~2 sits on a steep transition between $k = 1.5$
(fires $-1$ on $87\%$ of observations) and $k = e$ (fires on
$2.4\%$).  The default $k = e$ lands in a near-inactive regime,
which is consistent with the v2only ablation
(Appendix~\ref{app:exp:v2only}).  A choice of $k$ between $2.0$ and
$e$ would make Vote~2 a materially more active contributor; whether
this improves accuracy is not determined by counterfactual replay.
Third, the EMA time-constant derivation
(Appendix~\ref{app:vote2}, Lemma~\ref{lem:ema_tau}) pins $k$ to $e$
via $\beta_2^{\tau} = e^{-1}$; any alternate choice would need a
corresponding derivation.  At the derived value, Vote~2 behaves
conservatively, consistent with its role as a second diagnostic
rather than a primary driver.

\paragraph{Divisor.}
The existing Table~\ref{tab:divisor_comparison}
(Appendix~\ref{app:exponent}) gives the \emph{worst-case} per-level
convergence degradation $\cosh(4\beta_1/d)$ abstractly.  We complement
this with the \emph{empirically realised mean} degradation
$\mathbb{E}_s[\cosh(s \beta_1 / d)]$, averaged over the compound
level distribution observed across all 8 seeds of cadence-8 LionVote on
ViT-Tiny/CIFAR-100 ($n = 44992$ per-voting-epoch observations,
Table~\ref{tab:compound_distribution}).

\begin{table}[H]
\centering
\small
\caption{Realised mean convergence-bound degradation by divisor
($\beta_1 = 0.9$, compound levels from 8-seed pooled cadence-8 run).
The realised mean is a weighted average of $\cosh(s \beta_1 / d)$
using the empirical distribution over~$s$
(Table~\ref{tab:compound_distribution}).}
\label{tab:divisor_empirical}
\begin{tabular}{c c c}
\toprule
Divisor $d$ & Realised mean $\mathbb{E}_s[\cosh(s\beta_1/d)]$
  & Worst case $\cosh(4\beta_1/d)$ \\
\midrule
$1.5$ & $1.96\;(+96\%)$ & $5.56$ \\
$\mathbf{2.0}$ (default) & $\mathbf{1.47\;(+47\%)}$ & $\mathbf{3.11}$ \\
$2.5$ & $1.28\;(+28\%)$ & $2.23$ \\
$3.0$ & $1.19\;(+19\%)$ & $1.81$ \\
\bottomrule
\end{tabular}
\end{table}

\begin{table}[H]
\centering
\small
\caption{Empirical distribution of compound level $s$ across all
8 seeds of cadence-8 LionVote on ViT-Tiny/CIFAR-100
($n = 44992$ per-voting-epoch observations).}
\label{tab:compound_distribution}
\begin{tabular}{r c c c c c c c c c}
\toprule
$s$
  & $-4$ & $-3$ & $-2$ & $-1$ & $0$ & $+1$ & $+2$ & $+3$ & $+4$ \\
\midrule
Fraction
  & $9.63\%$ & $9.33\%$ & $12.34\%$ & $19.17\%$
  & $32.24\%$ & $9.31\%$ & $3.85\%$ & $2.04\%$ & $2.08\%$ \\
\bottomrule
\end{tabular}
\end{table}

At $d = 1.5$ the realised mean degradation is ${\sim}1.3{\times}$ that
at $d = 2$; this quantifies, on realised rather than worst-case levels,
the claim in Appendix~\ref{app:exponent} that $d = 2$ sits near
the centre of a defensible interval.  At $d = 3$, the per-level step
is small enough ($1.35{\times}$ per level) that the mechanism has
limited expressivity in the observed level range.  The
$|s| = 4$ saturation rate of $11.7\%$ indicates that the default
$L = 4$ is actively used, and lowering to $L = 3$ would truncate
approximately $12\%$ of per-voting-epoch decisions.

\paragraph{Tiebreaker threshold.}
The $1\%$ threshold is calibrated to the per-validation-sample
coefficient of variation at typical validation set sizes
(Appendix~\ref{app:tiebreaker}), and is explicitly described there as
``determined to order of magnitude rather than uniquely''.  We do not
sweep this threshold because the tiebreaker produces a single global
scalar per voting epoch ($37$ decisions at cadence~8), so the
counterfactual-vote-composition exercise used above for Votes~1 and~2
would not be informative at the per-parameter level.

\paragraph{What this does and does not show.}
Counterfactual vote composition characterises the threshold's
\emph{operating regime} -- whether the method sits near a phase
transition, and which vote components are active versus dormant.
It does \emph{not} predict accuracy under retraining at alternate
thresholds.  Vote~1 appears robust near its default; Vote~2 at the
derived $k = e$ is near-dormant on this workload
(consistent with the v2only ablation); the divisor $d = 2$ is
defensible on both worst-case and realised-mean grounds.

\subsection{Vote 2 Only: Compound Level Analysis}
\label{app:exp:v2only}

The v2only variant disables Vote~1 and the tiebreaker, leaving only
the momentum health vote.  On ViT-Tiny/CIFAR-100 (seed~2, cadence~4),
compound level statistics at each voting epoch show that Vote~2
rarely fires:

\medskip
\begin{center}
\small
\begin{tabular}{r r r r r r r}
\toprule
Epoch & Mean & Min & Max & \#pos & \#neg & \#zero \\
\midrule
4   & $-0.007$ & $-1$ & 0 & 0 &   1 & 151 \\
64  & $-0.086$ & $-2$ & 0 & 0 &  10 & 142 \\
124 & $-0.086$ & $-3$ & 0 & 0 &  10 & 142 \\
184 & $-0.053$ & $-2$ & 0 & 0 &   7 & 145 \\
244 & $-0.046$ & $-3$ & 0 & 0 &   5 & 147 \\
300 & $+0.000$ &   0  & 0 & 0 &   0 & 152 \\
\bottomrule
\end{tabular}
\end{center}
\medskip

At every epoch, ${\ge}89\%$ of parameters remain at level~0.  No
parameter ever reaches a positive level.  The momentum-to-gradient
ratio $r = \|m\| / \|\bar{g}\|$ stays within the dead zone
$[1/e, e]$ for the vast majority of parameters, consistent with the
EMA operating near stationarity.  v2only approximates Lion because
Vote~2 rarely fires under these training dynamics
(\S\ref{sec:experiments:ablation}).

\subsection{Vote Dynamics Over Training}
\label{app:vote_dynamics}

Appendix~\ref{app:exp:v2only} established that Vote~2 rarely fires.
This subsection characterises the joint firing behaviour of Votes~1
and~2 across training, including the resulting tiebreaker activation
rate.  Data are from the plain-Lion diagnostic run used in
Appendix~\ref{app:exp:raw_diagnostics}
($n = 5472$: $36$ voting epochs $\times$ $152$ parameters,
ViT-Tiny/CIFAR-100, seed~2, cadence~8).

\paragraph{Joint vote distribution.}
Applying the paper's default thresholds
($\cos > 0.5$, $r > e$ or $r < 1/e$) to the recorded diagnostics
yields the joint distribution in
Table~\ref{tab:vote_cross_tab}.

\begin{table}[H]
\centering
\small
\caption{Joint distribution of Vote~1 and Vote~2 outcomes under
default thresholds (Lion baseline, ViT-Tiny/CIFAR-100, seed~2,
$n = 5472$).}
\label{tab:vote_cross_tab}
\begin{tabular}{c c c c c}
\toprule
 & $V_2 = -1$ & $V_2 = 0$ & $V_2 = +1$ & row total \\
\midrule
$V_1 = -1$ & $1.06\%$ & $36.49\%$ & $0.00\%$ & $37.55\%$ \\
$V_1 =  0$ & $1.33\%$ & $58.64\%$ & $0.00\%$ & $59.98\%$ \\
$V_1 = +1$ & $0.04\%$ & $2.43\%$ & $0.00\%$ & $2.47\%$ \\
\midrule
col.\ total & $2.43\%$ & $97.57\%$ & $0.00\%$ & \\
\bottomrule
\end{tabular}
\end{table}

Three facts stand out.  First, $V_2 = +1$ never fires
(momentum norm never falls below gradient norm by a factor of~$e$
on any observation).  Second, $V_1$ and $V_2$ never take strictly
opposite non-zero values in the dominant sense
($(V_1, V_2) = (-1, +1)$ is $0\%$, and
$(V_1, V_2) = (+1, -1)$ is $0.04\%$ -- two observations).
Third, and most consequentially,
$V_1 + V_2 = 0$ holds for $58.68\%$ of per-parameter decisions --
overwhelmingly via the $(0, 0)$ cell ($58.64\%$), with active
disagreement contributing only $0.04\%$.  On this workload, the
tiebreaker is therefore eligible for a majority of per-parameter
voting decisions, and Vote~1 alone determines the outcome for the
remaining $41\%$ (Vote~2 is decisive in isolation only when $V_1$
abstains and $V_2$ fires, which is $1.33\%$).

\paragraph{Per-voting-epoch firing.}
Table~\ref{tab:vote_time_series} sweeps $V_1$ and $V_2$ firing
rates across training at nine representative epochs (full $36$-epoch
series available in supplementary CSV).

\begin{table}[H]
\centering
\small
\caption{Per-voting-epoch firing rates, default thresholds
(Lion baseline, seed~2, ViT-Tiny/CIFAR-100; $152$ parameter tensors
per row).  Early epochs: $V_1 = 0$ dominates; late epochs:
$V_1 = -1$ dominates.  $V_2$ fires sporadically throughout with
no monotonic trend.}
\label{tab:vote_time_series}
\begin{tabular}{r c c c c}
\toprule
Epoch & $V_1 = +1$ & $V_1 = 0$ & $V_1 = -1$ & $V_2 = -1$ \\
\midrule
$16$  & $4.61\%$  & $75.66\%$ & $19.74\%$ & $0.00\%$ \\
$48$  & $2.63\%$  & $68.42\%$ & $28.95\%$ & $3.29\%$ \\
$96$  & $2.63\%$  & $73.68\%$ & $23.68\%$ & $6.58\%$ \\
$144$ & $5.26\%$  & $53.95\%$ & $40.79\%$ & $0.66\%$ \\
$192$ & $2.63\%$  & $55.92\%$ & $41.45\%$ & $7.24\%$ \\
$216$ & $1.32\%$  & $44.08\%$ & $54.61\%$ & $3.29\%$ \\
$248$ & $0.66\%$  & $46.71\%$ & $52.63\%$ & $2.63\%$ \\
$272$ & $1.32\%$  & $45.39\%$ & $53.29\%$ & $2.63\%$ \\
$296$ & $0.66\%$  & $43.42\%$ & $55.92\%$ & $0.00\%$ \\
\bottomrule
\end{tabular}
\end{table}

$V_1 = -1$ (reversal detection) rises nearly monotonically from
$20\%$ in early training to $56\%$ by epoch~$296$: as training
progresses, epoch-to-epoch gradient direction stability decreases
for an increasing fraction of layers.  $V_1 = +1$ (strong
stability) correspondingly decays from $5\%$ to $1\%$.  $V_2$ shows
no monotonic trend; its firings appear episodically at
single-digit rates.  This profile is what drives the downward
level migration visible in the stacked area plot
(Figure~\ref{fig:level_flow}).

\paragraph{Relation to level changes.}
Combined with the compound level dynamics (Figure~\ref{fig:compound},
Table~\ref{tab:compound_vit}), the picture is: the dominant driver of
negative compound level accumulation is Vote~1's growing reversal
detection rate, moderated by the tiebreaker whenever both local
votes abstain ($58.6\%$ of decisions) -- typically resolving in
favour of rate maintenance or reduction as validation loss ceases
to improve monotonically.

\subsection{Additional Compound Level Trajectories}
\label{app:exp:compound}

\paragraph{ViT-Tiny/CIFAR-100, cadence 8.}
\label{tab:c8_vit}

\medskip
\begin{center}
\small
\begin{tabular}{l r r r r r r r}
\toprule
Layer type & ep8 & ep64 & ep120 & ep176 & ep232 & ep296 & Mult.\ \\
\midrule
all\_attn       & $-0.35$ & $+0.06$ & $-1.22$ & $-1.75$ & $-2.13$ & $-2.06$ & $0.40\times$ \\
all\_mlp        & $-0.35$ & $+0.12$ & $-0.83$ & $-1.55$ & $-1.98$ & $-1.98$ & $0.41\times$ \\
all\_norm       & $-0.48$ & $+0.11$ & $-0.36$ & $-0.53$ & $-1.22$ & $-1.43$ & $0.53\times$ \\
embed+patch     & $-0.28$ & $+0.12$ & $-1.28$ & $-1.19$ & $-2.19$ & $-2.06$ & $0.40\times$ \\
head+final\_norm & $-0.25$ & $+0.88$ & $-0.31$ & $-0.50$ & $-1.44$ & $-1.19$ & $0.59\times$ \\
\bottomrule
\end{tabular}
\end{center}
\medskip

At cadence~8, the early-training trajectory differs: levels are near
zero or slightly positive at epoch~64, indicating the system initially
finds no strong signal for rate reduction with the longer voting
window.  By mid-training (epoch~120), differentiation emerges, and
by epoch~296 the attention-vs-normalisation spread ($0.63$ levels,
LR ratio $1.33$) is comparable to cadence~4 ($0.61$ levels, ratio
$1.32$).

\paragraph{WRN-28-10/CIFAR-10, cadence 4.}

\medskip
\begin{center}
\small
\begin{tabular}{l r r r r r r}
\toprule
Layer group & ep4 & ep44 & ep84 & ep124 & ep164 & ep200 \\
\midrule
conv0     & $-0.12$ & $-3.12$ & $-3.00$ & $-0.88$ & $-1.50$ & $-1.00$ \\
layer1    & $+0.05$ & $-2.37$ & $-2.38$ & $-2.13$ & $-2.00$ & $-1.68$ \\
layer2    & $-0.02$ & $-2.41$ & $-2.67$ & $-2.13$ & $-1.83$ & $-1.62$ \\
layer3    & $-0.06$ & $-1.98$ & $-2.54$ & $-1.20$ & $-1.48$ & $-1.84$ \\
fc        & $-0.50$ & $-3.38$ & $-2.06$ & $-2.12$ & $-1.31$ & $-1.75$ \\
final\_bn & $+0.31$ & $-3.00$ & $-2.00$ & $-1.75$ & $-1.19$ & $-1.81$ \\
\bottomrule
\end{tabular}
\end{center}
\medskip

The three repeated block groups (layer1, layer2, layer3) show a spread
of $0.22$ levels at epoch~200, compared to $0.61$ levels between
attention and normalisation on ViT-Tiny.  The smaller spread reflects
the architectural uniformity of WideResNet's residual groups.

\subsection{Effective Learning Rate Fingerprint}
\label{app:exp:lr_fingerprint}

\begin{figure}[h]
\centering
\includegraphics[width=0.55\textwidth]{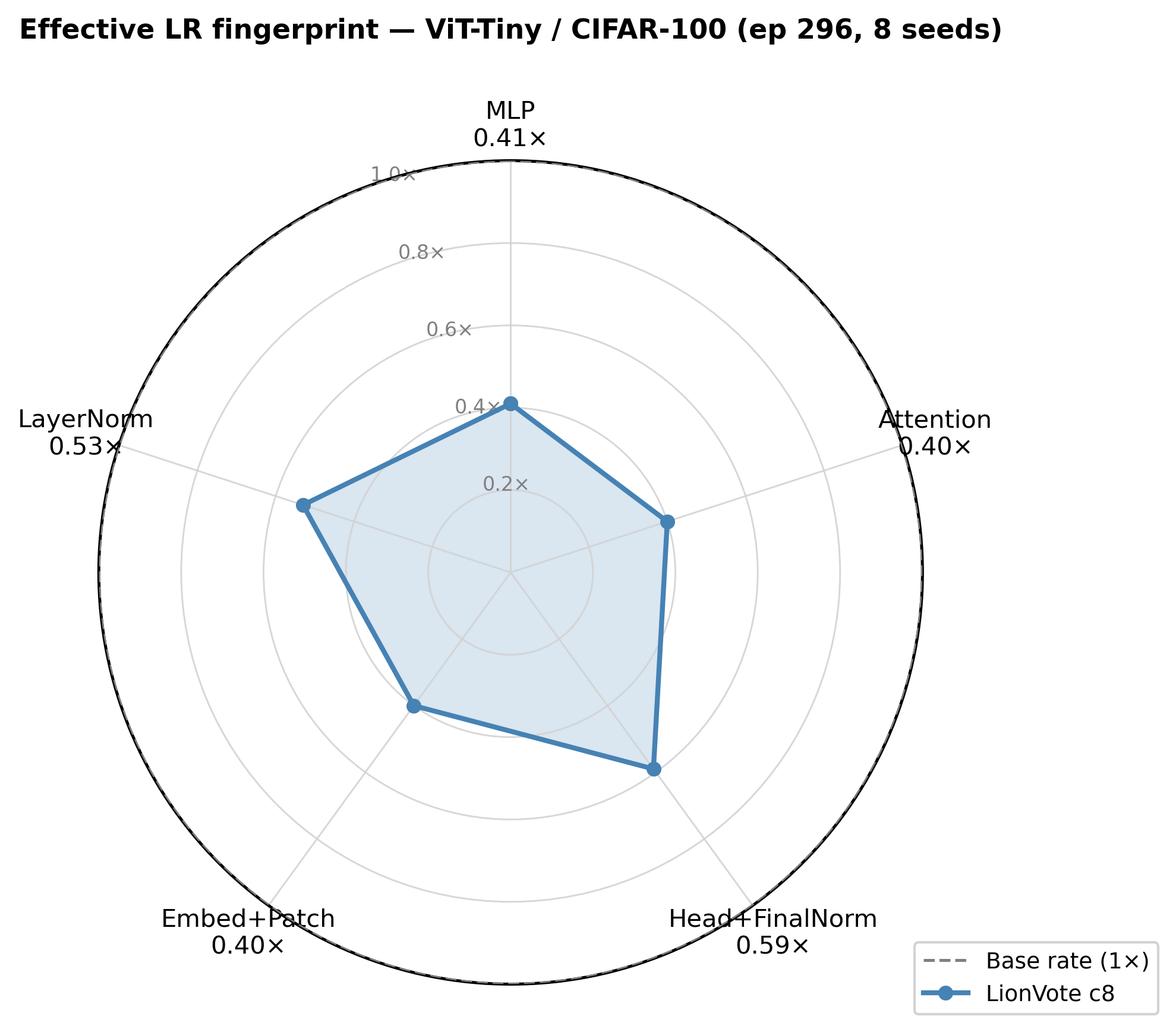}
\caption{Effective LR+WD fingerprint for LionVote (cadence~8)
on ViT-Tiny/CIFAR-100 at epoch~296 (8-seed average).  Each axis shows
the compound multiplier $\alpha_i/\mathrm{lr}$ as a fraction of the
base rate.  The dashed circle marks the base rate ($1{\times}$).  Head
and final normalisation layers retain the highest multiplier
($0.59{\times}$, from compound level $-1.19$, Table~\ref{tab:c8_vit});
attention, MLP, and embedding layers are reduced to ${\sim}0.40{\times}$.}
\label{fig:radar}
\end{figure}

\subsection{Compound Level Mass Flow}
\label{app:exp:level_flow}

\begin{figure}[H]
\centering
\includegraphics[width=\textwidth]{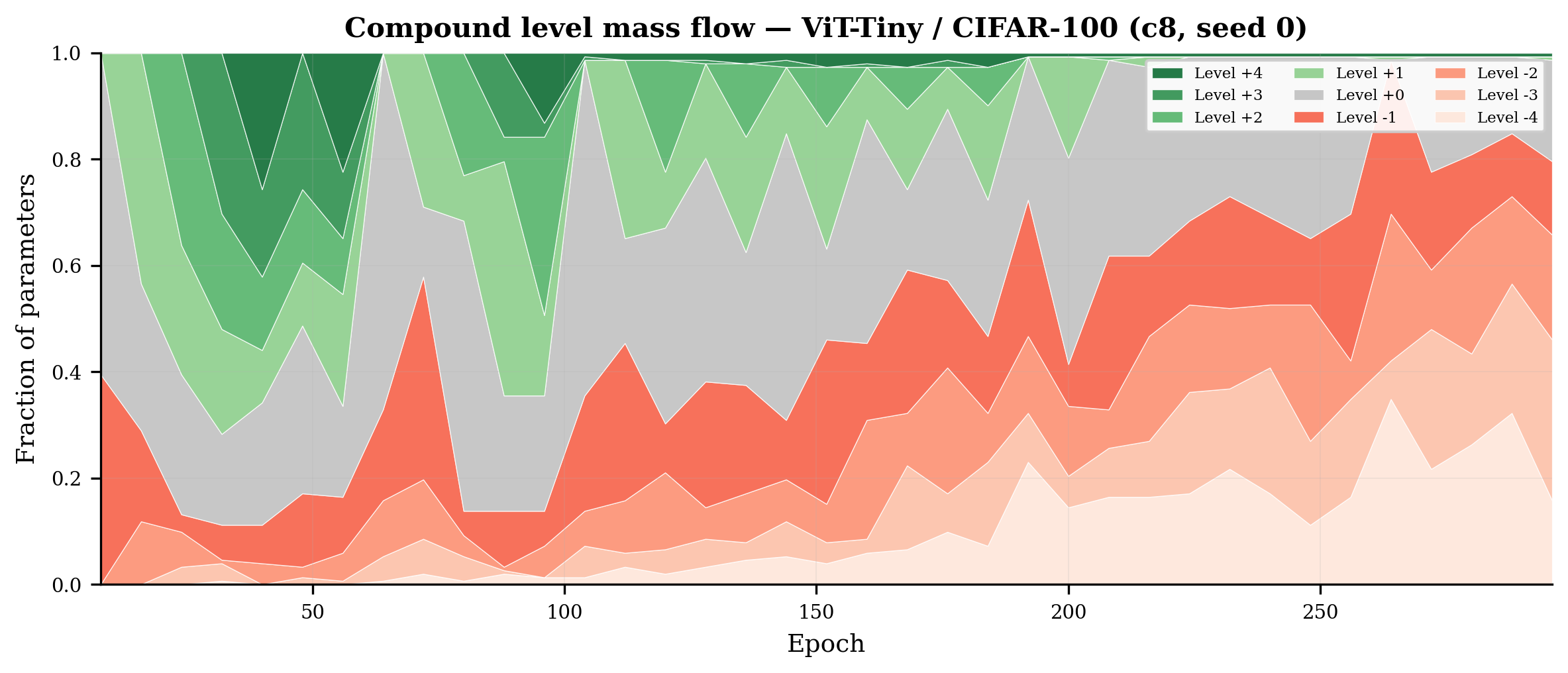}
\caption{Stacked area plot showing the fraction of parameters at each
compound level over training (ViT-Tiny/CIFAR-100, cadence~8, seed~2).
Green shades indicate positive levels; red shades indicate negative
levels; grey is level~0.  The mass migrates from positive/neutral
levels early in training toward increasingly negative levels, with
most parameters settling between levels $-1$ and $-3$ by epoch~250.}
\label{fig:level_flow}
\end{figure}

\subsection{Cadence Effect on Layer-Type Differentiation}
\label{app:exp:cadence_comparison}

\begin{figure}[H]
\centering
\includegraphics[width=\textwidth]{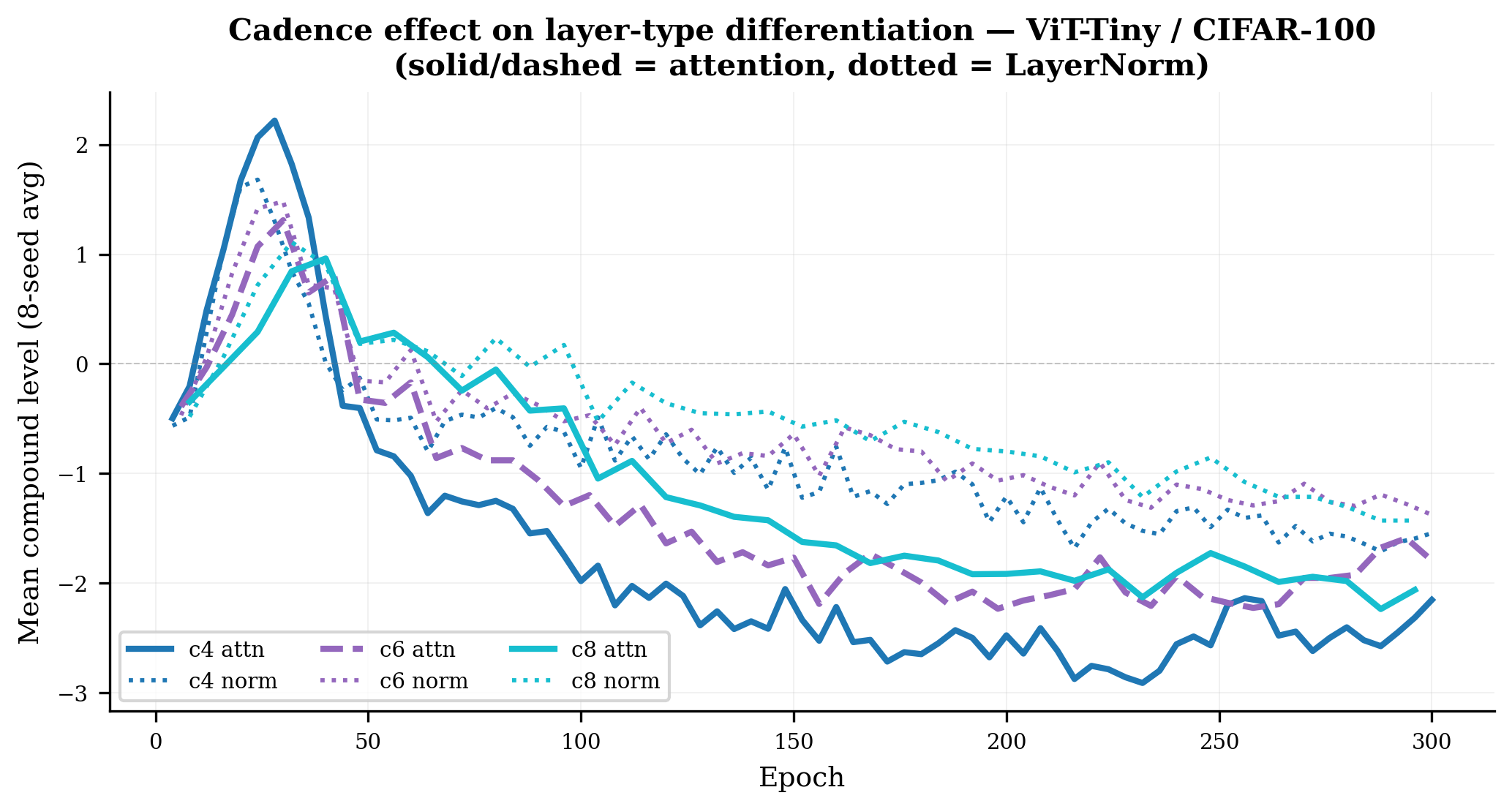}
\caption{Mean compound level for attention (solid/dashed) and
LayerNorm (dotted) parameters across cadences $c \in \{4, 6, 8\}$ on
ViT-Tiny/CIFAR-100 (8-seed average).  Lower cadence produces faster
initial differentiation but noisier long-term trajectories.  All
cadences converge to similar attention--normalisation spreads by
epoch~250, though $c{=}4$ overshoots early in training.}
\label{fig:cadence_comparison}
\end{figure}

\subsection{Effective Learning Rate Dynamics}
\label{app:exp:effective_lr}

\begin{figure}[H]
\centering
\includegraphics[width=\textwidth]{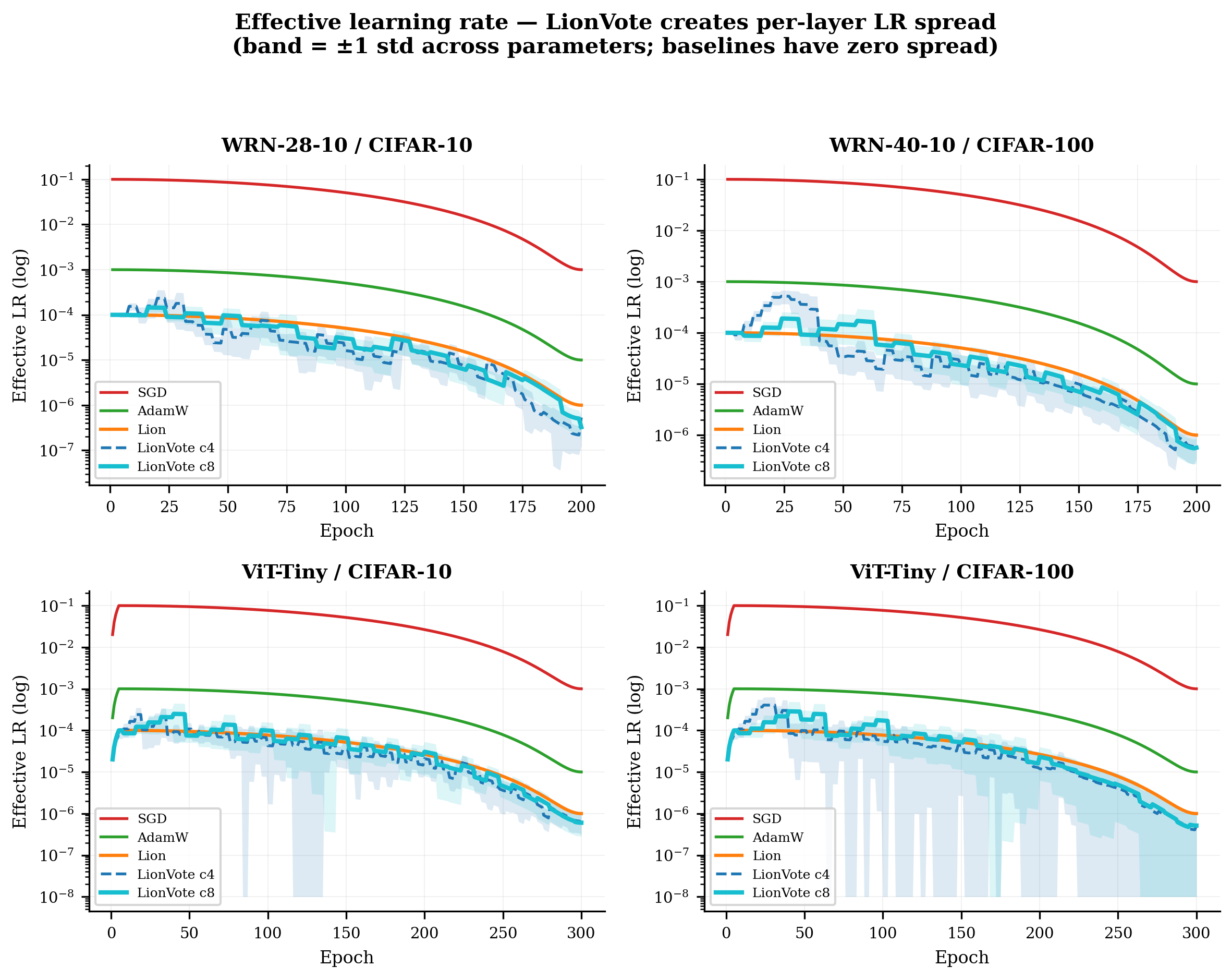}
\caption{Effective learning rate (log scale) for all configurations
and optimizers.  Bands for LionVote show $\pm 1$ std across layers within each seed, reflecting per-layer LR spread; baselines have zero
spread (single global rate).  LionVote's per-layer rates span
approximately one order of magnitude, with the spread widening as
training progresses.}
\label{fig:effective_lr}
\end{figure}

\subsection{Training Loss Curves}
\label{app:exp:Tloss_curve}

\begin{figure}[H]
\centering
\includegraphics[width=\textwidth]{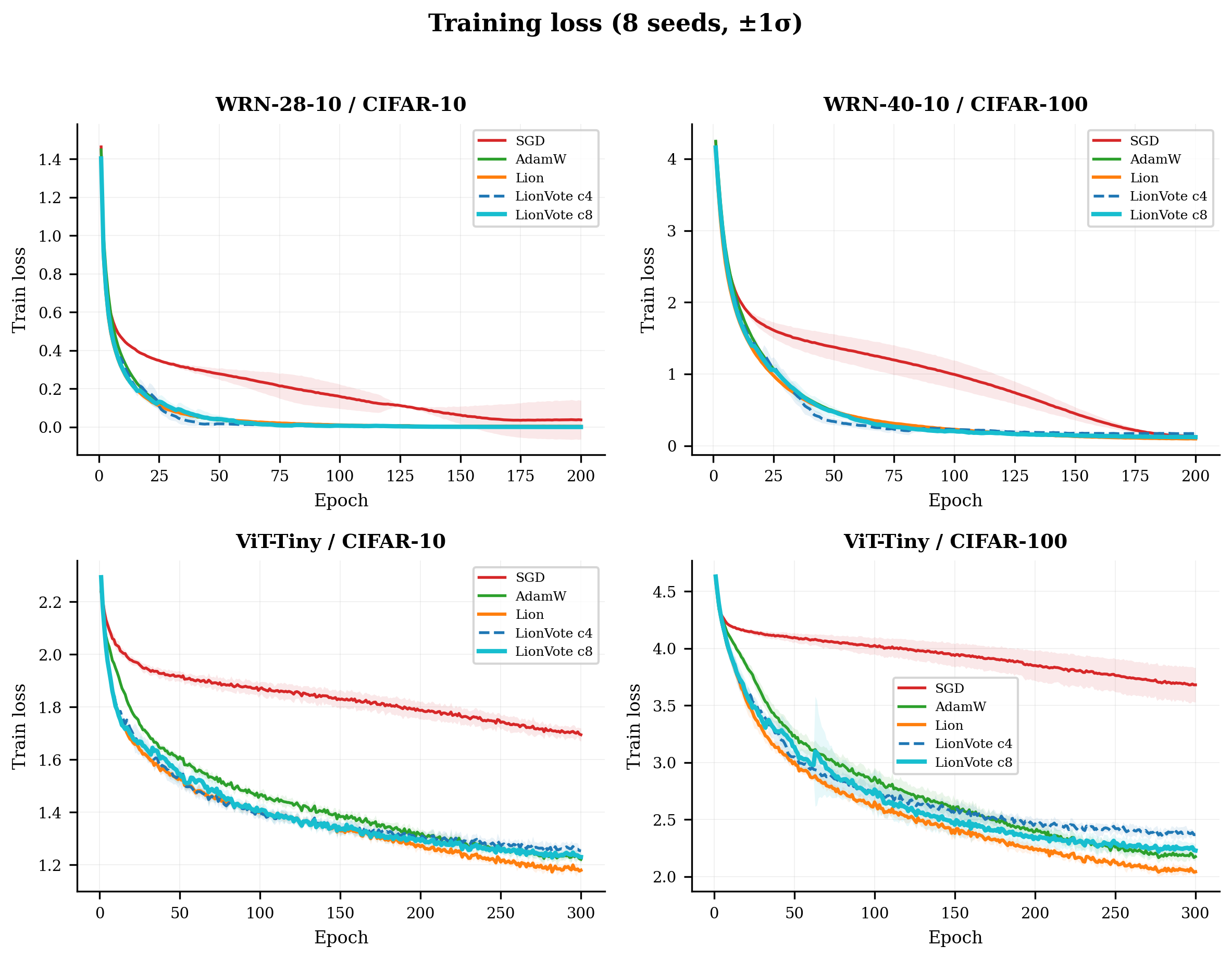}
\caption{Training loss for all configurations (mean $\pm 1$ std, 8
seeds).  Adaptive optimizers (AdamW, Lion, LionVote) converge
substantially faster than SGD on ViT-Tiny.  On WRN configurations,
SGD achieves the lowest final training loss despite slower initial
descent.}
\label{fig:train_loss}
\end{figure}

\subsection{Validation Loss Curves}
\label{app:exp:Vloss_curves}

\begin{figure}[H]
\centering
\includegraphics[width=\textwidth]{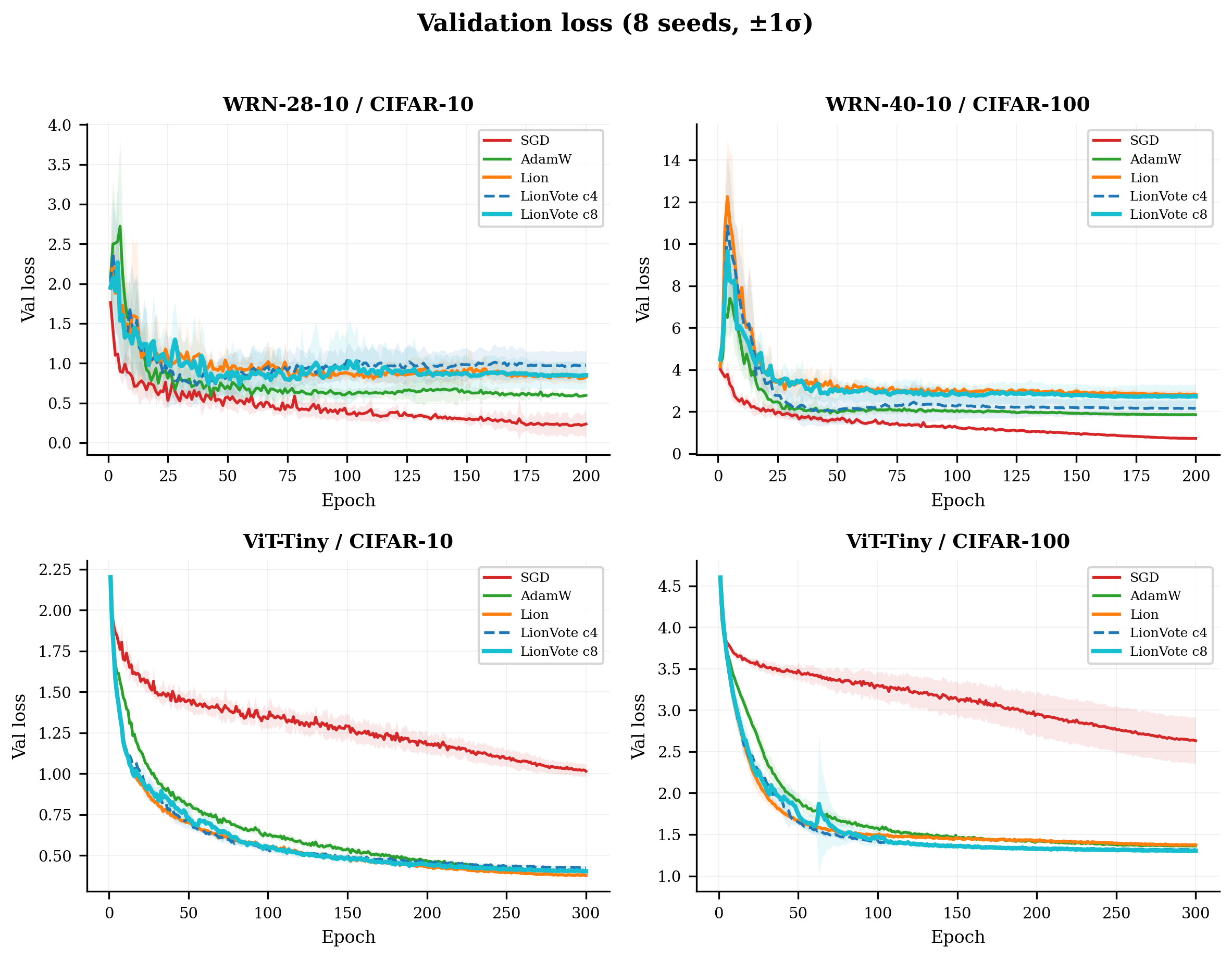}
\caption{Validation loss for all configurations (mean $\pm 1$ std, 8
seeds).  On WRN configurations, SGD achieves the lowest validation
loss.  On ViT-Tiny, all adaptive methods converge to similar
validation loss, with LionVote showing slightly faster early descent
on CIFAR-100.}
\label{fig:val_loss}
\end{figure}

\subsection{Generalisation Gap}
\label{app:exp:gen_gap}

\begin{figure}[H]
\centering
\includegraphics[width=\textwidth]{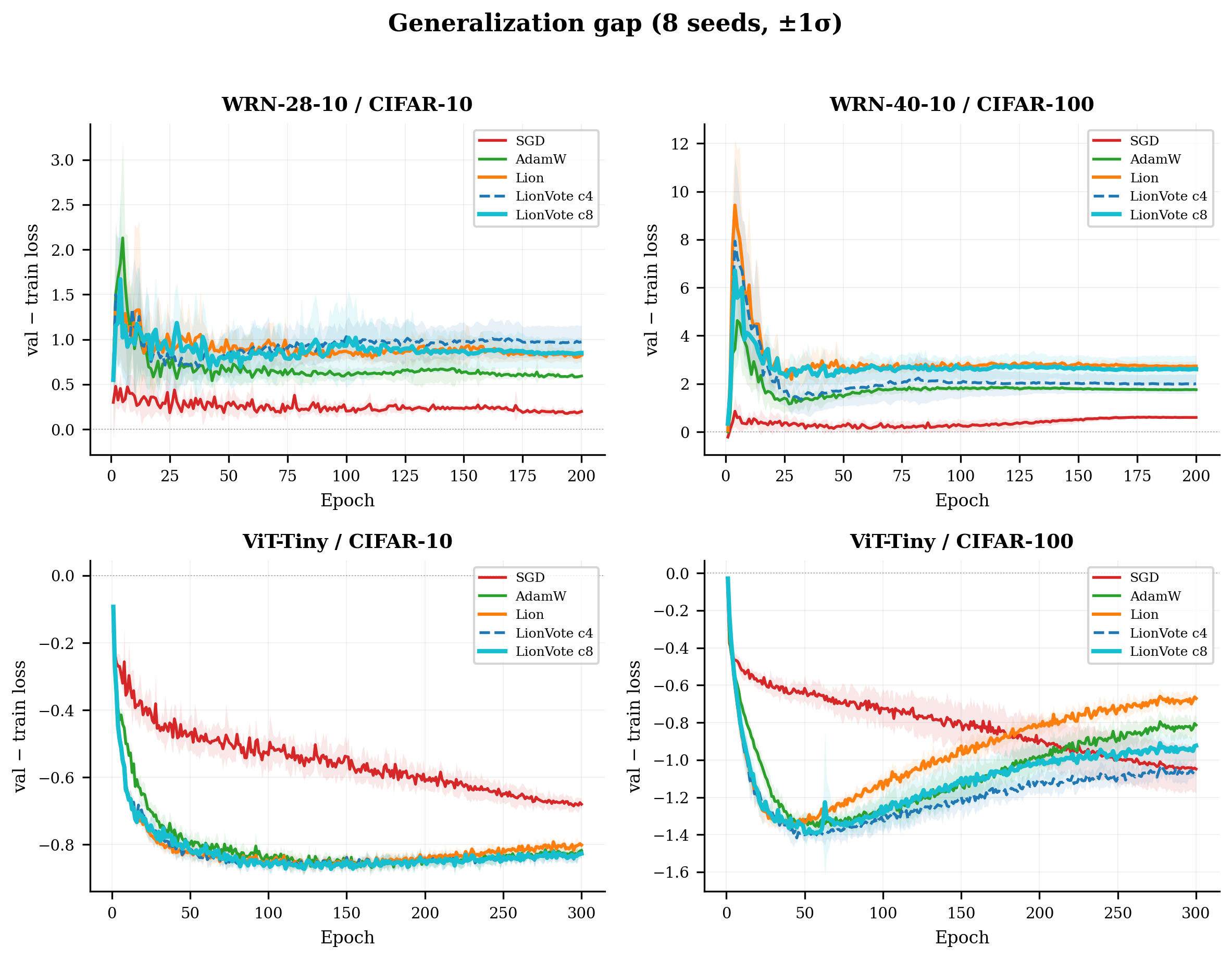}
\caption{Generalisation gap (validation loss $-$ training loss) for
all configurations (mean $\pm 1$ std, 8 seeds).  On WRN, SGD
maintains the smallest gap, indicating less overfitting.  On
ViT-Tiny, all methods show negative gaps (validation loss $<$
training loss), a consequence of Mixup, CutMix, and label smoothing
augmentations that make the training objective harder than the
clean validation objective.}
\label{fig:gen_gap}
\end{figure}

\subsection{Control Experiments: Static Multipliers and Lower Base Rate}
\label{app:exp:controls}

Four control baselines on ViT-Tiny/CIFAR-100 (8 seeds each, cosine
warmup, 300 epochs) test whether the per-layer differentiation
discovered by LionVote can be replicated by static means or by
adjusting global hyperparameters.

\paragraph{Lion-Typed.}  Lion with fixed per-layer-type LR multipliers
from epoch~1, using the exact epoch-300 values from
Table~\ref{tab:compound_vit} (cadence~4): attention ${\times}0.38$,
MLP ${\times}0.36$, normalisation ${\times}0.50$, embedding
${\times}0.40$.  Base $\mathrm{lr}{=}10^{-4}$, $\mathrm{wd}{=}0.5$.

\paragraph{Lion ($\mathbf{lr{=}4{\times}10^{-5}}$).}  Plain Lion at
${\sim}2.5{\times}$ lower base rate, matching LionVote's average
effective rate at convergence.  $\mathrm{wd}{=}0.5$.

\paragraph{Weight decay sweep.}  Lion at $\mathrm{wd}{=}0.25$
($0.5{\times}$ default) and $\mathrm{wd}{=}0.75$ ($1.5{\times}$
default), both with $\mathrm{lr}{=}10^{-4}$.

\medskip
All $p$-values are Welch two-sample $t$-tests ($n = 8$ per group).
\begin{center}
\small
\begin{tabular}{l c c c}
\toprule
Method & Accuracy (\%) & $\Delta$ vs.\ Lion & Welch $p$ vs.\ Lion \\
\midrule
LionVote (c8) & $69.71 \pm 0.60$ & $+0.76$ & $0.012$ \\
Lion ($\mathrm{wd}{=}0.25$) & $69.12 \pm 0.58$ & $+0.17$ & $0.52$ \\
Lion ($\mathrm{lr}{=}10^{-4}$) & $68.95 \pm 0.41$ & --- & --- \\
Lion ($\mathrm{wd}{=}0.75$) & $68.91 \pm 0.31$ & $-0.04$ & $0.84$ \\
Lion ($\mathrm{lr}{=}4{\times}10^{-5}$) & $68.72 \pm 0.45$ & $-0.23$ & $0.30$ \\
Lion-Typed & $68.47 \pm 0.34$ & $-0.48$ & $0.024$ \\
\bottomrule
\end{tabular}
\end{center}
\medskip

Lion-Typed is significantly worse than standard Lion ($-0.48$~pp,
$p = 0.024$) but not significantly worse than AdamW ($-0.28$~pp,
$p = 0.15$).
The lower-LR control is statistically indistinguishable from Lion
($p = 0.30$) and from AdamW ($p = 0.89$).  The two controls are
indistinguishable from each other ($p = 0.23$), and both are
significantly below LionVote at cadence~8 ($p = 0.003$ and
$p < 0.001$ respectively).
Neither weight decay variant significantly differs from the default:
$\mathrm{wd}{=}0.25$ is the best static variant but remains
$0.59$~pp below LionVote at cadence~8 ($p = 0.075$).

\paragraph{Per-seed accuracy.}

\medskip
\begin{center}
\small
\begin{tabular}{r c c c}
\toprule
Seed & Lion-Typed & Lion ($4{\times}10^{-5}$) & LionVote (c8) \\
\midrule
 2 & 68.60 & 68.81 & 70.44 \\
18 & 68.15 & 68.96 & 69.54 \\
20 & 68.01 & 68.09 & 69.29 \\
26 & 68.92 & 68.88 & 69.89 \\
40 & 68.84 & 69.38 & 68.95 \\
44 & 68.46 & 68.35 & 70.84 \\
48 & 68.66 & 69.06 & 69.26 \\
52 & 68.16 & 68.24 & 69.45 \\
\bottomrule
\end{tabular}
\end{center}
\medskip

No Lion-Typed seed exceeds $69\%$; one lower-LR seed (of eight)
exceeds $69\%$; seven of eight LionVote seeds exceed $69\%$.

\paragraph{Convergence speed.}

\medskip
\begin{center}
\small
\begin{tabular}{l c c c c c}
\toprule
Threshold & Lion ($4{\times}10^{-5}$) & Lion-Typed & Lion ($10^{-4}$) & LV c8 & AdamW \\
\midrule
$63\%$ &  68 &  71 &  85 &  76 & 121 \\
$65\%$ & 103 & 104 & 130 & 100 & 158 \\
$67\%$ & 187 & 188 & 199 & 133 & 214 \\
$68\%$ & 240 & 253 & 239 & 160 & 250 \\
$69\%$ & 2/8 & 0/8 & 3/8 & 7/8 & 3/8 \\
\bottomrule
\end{tabular}
\end{center}
\medskip

Both controls converge faster than Lion ($10^{-4}$) and AdamW to
early thresholds ($63$--$67\%$) but plateau lower.  LionVote at
cadence~8 reaches $68\%$ at epoch~$160$, a full $80$ epochs before
the lower-LR control.

\paragraph{Per-seed accuracy: ablation variants and WD sweep (ViT/C100, cadence~8).}

\medskip
\begin{center}
\small
\begin{tabular}{r c c c c c}
\toprule
Seed & v1only & notie & symmetric & Lion (wd=0.25) & Lion (wd=0.75) \\
\midrule
 2 & 68.86 & 68.54 & 70.26 & 69.43 & 69.15 \\
18 & 69.21 & 69.35 & 71.27 & 69.67 & 69.28 \\
20 & 68.03 & 68.17 & 70.43 & 68.35 & 68.93 \\
26 & 67.79 & 67.65 & 68.86 & 68.53 & 68.49 \\
40 & 69.11 & 68.94 & 70.69 & 70.02 & 69.22 \\
44 & 68.42 & 68.81 & 72.15 & 69.18 & 69.02 \\
48 & 68.55 & 68.81 & 70.49 & 69.08 & 68.54 \\
52 & 68.57 & 68.85 & 70.18 & 68.68 & 68.68 \\
\midrule
Mean & 68.57 & 68.64 & 70.54 & 69.12 & 68.91 \\
Std  &  0.49 &  0.52 &  0.94 &  0.58 &  0.31 \\
\bottomrule
\end{tabular}
\end{center}
\medskip

\subsection{Raw Lion Gradient Diagnostics (No Voting Mechanism)}
\label{app:exp:raw_diagnostics}

To test whether the layer-type differentiation observed in
\S\ref{sec:analysis:layers} is an artefact of LionVote's specific
voting mechanism, the same per-layer diagnostics were measured during
standard Lion training on ViT-Tiny/CIFAR-100 (seed~2, cadence~8
measurement interval, cosine warmup, 300 epochs).  No voting, no
compound levels, no per-layer LR adjustment---only the raw diagnostic
signals.

\paragraph{Gradient direction stability.}
Table~\ref{tab:raw_diag} reports the mean epoch-mean cosine alignment
$c_i = \cos(\bar{g}_i^{(\mathrm{curr})}, \bar{g}_i^{(\mathrm{prev})})$
by layer type in late training (epoch $\geq 200$).

\begin{table}[H]
\centering
\small
\caption{Raw per-layer gradient diagnostics during standard Lion
training (ViT-Tiny/CIFAR-100, seed~2, epoch~$\geq 200$).
``Would-vote'' columns show the fraction of measurement epochs
at which the raw cosine alignment would trigger a Vote~1 outcome
under the paper's thresholds.}
\label{tab:raw_diag}
\begin{tabular}{l c c c c c}
\toprule
Layer type & Mean cosine & Std & Would vote $-1$ & Abstain & Would vote $+1$ \\
\midrule
attn & $-0.005$ & $0.106$ & $56\%$ & $43\%$ & $<1\%$ \\
mlp  & $-0.002$ & $0.054$ & $56\%$ & $44\%$ & $0\%$ \\
norm & $+0.104$ & $0.192$ & $30\%$ & $69\%$ & $1\%$ \\
\bottomrule
\end{tabular}
\end{table}

Under standard Lion, attention and MLP layers have mean cosine
alignment near zero: their epoch-mean gradients are essentially
uncorrelated across consecutive measurement windows.  Normalisation
layers have mean cosine $+0.104$---substantially higher, indicating
more directionally stable gradients.  Attention and MLP parameters
would receive a negative vote (cosine $< 0$) on $56\%$ of measurement
epochs, compared to $30\%$ for normalisation---a $1.9{\times}$ ratio.
This asymmetry in raw gradient statistics is the signal that LionVote's
Vote~1 detects; it exists in standard Lion training independent of the
voting mechanism.

\paragraph{Momentum health.}
The momentum-to-gradient norm ratio $r = \|m\| / \|\bar{g}\|$ stays
within the Vote~2 dead zone $[1/e, e]$ for $98$--$100\%$ of
measurements across all layer types, confirming that Vote~2's low
firing rate (\S\ref{sec:experiments:ablation}) reflects the training
dynamics, not the threshold placement.

\paragraph{Cross-optimizer comparison (Lion vs.\ AdamW).}
Running the same diagnostics under AdamW (ViT-Tiny/CIFAR-100,
cosine warmup, lr${=}10^{-3}$, wd${=}0.05$, single seed) reveals that
the attn~$<$~norm ordering is optimizer-general, but Lion amplifies the
gap.

\begin{table}[H]
\centering
\small
\caption{Cross-optimizer gradient diagnostics (ViT-Tiny/CIFAR-100,
seed~2, late training epoch~$\geq 200$, 13 measurement epochs
$\times$ 152 parameters).
Negative-vote rate: fraction of (param, epoch) pairs with
cosine~$<~0$.}
\label{tab:cross_opt_diag}
\begin{tabular}{l c c c c}
\toprule
& \multicolumn{2}{c}{Mean cosine $\pm$ std} & \multicolumn{2}{c}{Negative-vote rate} \\
\cmidrule(lr){2-3}\cmidrule(lr){4-5}
Layer type & Lion & AdamW & Lion & AdamW \\
\midrule
attn  & $-0.005 \pm 0.106$ & $-0.005 \pm 0.097$ & $56.4\%$ & $56.1\%$ \\
mlp   & $-0.002 \pm 0.054$ & $+0.001 \pm 0.056$ & $56.4\%$ & $54.3\%$ \\
norm  & $+0.104 \pm 0.192$ & $+0.053 \pm 0.161$ & $29.8\%$ & $37.5\%$ \\
other & $+0.081 \pm 0.216$ & $+0.027 \pm 0.183$ & $30.8\%$ & $55.8\%$ \\
\midrule
norm--attn gap & $0.109$ & $0.058$ & & \\
\bottomrule
\end{tabular}
\end{table}

The attn~$<$~norm ordering holds for both optimizers ($100\%$ of late
epochs under Lion, $92\%$ under AdamW), confirming the heterogeneity
is architectural.  Lion's $\operatorname{sign}(\cdot)$ amplifies the
norm--attn cosine gap $1.88{\times}$ relative to AdamW ($0.109$
vs.\ $0.058$): because sign discards gradient magnitude, directionally
unstable layers cannot compensate through large gradients.  AdamW's
second moment ($1/\sqrt{\hat{v}}$) implicitly down-weights noisy
coordinates, partially absorbing the layer-type difference.  This
explains why per-layer correction has the largest marginal value for
sign-based optimizers: no within-update self-correction mechanism
exists.

\paragraph{Vote 2 under AdamW: thresholds are Lion-specific.}
The cosine-alignment comparison above concerns only Vote~1's input
signal.  The analogous comparison for Vote~2 -- the
momentum-to-gradient norm ratio $r = \|m\|/\|\bar{g}\|$ -- reveals a
qualitative difference.  Under Lion,
$r$~has geometric mean ${\sim}1.8$ across layer types and sits
within the dead zone $[1/e, e]$ for ${\sim}97.6\%$ of observations
(Table~\ref{tab:cross_opt_v2}).  Under AdamW with the same
parameters monitored, $r$~has geometric mean ${\sim}8$--$10$ and lies
above~$e$ for ${\sim}100\%$ of observations.

\begin{table}[H]
\centering
\small
\caption{Vote~2 input statistics under Lion vs.\ AdamW
(ViT-Tiny/CIFAR-100, seed~2, all $36$ voting epochs,
$n = 5472$ per optimizer).  Geometric mean
$r = \|m\|/\|\bar{g}\|$ is reported per layer type; overall
``$r > e$'' is the $V_2 = -1$ firing rate.}
\label{tab:cross_opt_v2}
\begin{tabular}{l r r r r}
\toprule
 & \multicolumn{2}{c}{Geo.\ mean $r$ by layer} &
    \multicolumn{2}{c}{Overall $V_2$} \\
\cmidrule(lr){2-3}\cmidrule(lr){4-5}
Layer / Stat & Lion & AdamW & Lion & AdamW \\
\midrule
attn   & $1.898$ & $9.384$ & \multirow{4}{*}{---} & \multirow{4}{*}{---}  \\
mlp    & $1.790$ & $8.022$ & & \\
norm   & $1.769$ & $8.818$ & & \\
other  & $1.914$ & $9.930$ & & \\
\midrule
$V_2 = +1$ rate & $0.00\%$ & $0.00\%$ & & \\
$V_2 = 0$ rate  & $97.57\%$ & $0.02\%$ & & \\
$V_2 = -1$ rate & $2.43\%$ & $99.98\%$ & & \\
\bottomrule
\end{tabular}
\end{table}

The ratio difference is a consequence of the different momentum
mechanics.  Lion's $\beta_2$ outer EMA
($\beta_2 = 0.99$ default) smooths the update buffer over roughly
$100$~steps, keeping its norm comparable to the current gradient.
AdamW's first-moment EMA with $\beta_1 = 0.9$ accumulates raw
gradients on a much longer effective horizon when unnormalised by
the second moment (which is applied separately by division inside
the update itself rather than to the buffer we monitor), so the
buffer norm significantly exceeds the current-gradient norm on
this workload.  Consequently, Vote~2's symmetric $[1/e, e]$ dead
zone -- which is derived from the EMA time-constant identity
$\beta_2^{\tau} = e^{-1}$ -- is calibrated against Lion's
$\beta_2$-EMA and produces a near-constant $V_2 = -1$ signal under
AdamW, making it useless for per-layer modulation without
recalibration.

\paragraph{Optimizer-agnostic versus Lion-specific components.}
Combining the Vote~1 comparison (near-invariant between Lion and
AdamW: $V_1 = -1$~rate $37.6\%$ vs.\ $35.4\%$; $V_1 = +1$~rate
$2.5\%$ vs.\ $1.1\%$) with the Vote~2 comparison
(optimizer-dependent: dead-zone occupancy $97.6\%$ vs.\ $0.02\%$),
the LionVote mechanism decomposes as follows:

\begin{table}[H]
\centering
\small
\caption{LionVote component portability to other momentum-based
optimizers, based on the Lion/AdamW diagnostic comparison
(this subsection and
Table~\ref{tab:cross_opt_diag}).}
\label{tab:portability}
\begin{tabular}{l p{0.60\textwidth}}
\toprule
Component & Portability notes \\
\midrule
$V_1$ input (cosine of epoch-mean gradients)
  & Optimizer-agnostic: gradient-only statistic. \\
$V_1$ upper threshold $0.5$
  & Sheppard derivation (App.~\ref{app:vote1}) uses the sign-update
    structure; the $2/3$ convention is optimizer-agnostic. \\
$V_1$ lower threshold $0$
  & Optimizer-agnostic: parameter-free geometric identity. \\
$V_2$ input $\|m\|/\|\bar{g}\|$
  & Optimizer-agnostic: computable from any momentum-based optimizer. \\
$V_2$ thresholds $1/e$, $e$
  & Lion-specific: calibrated to Lion's $\beta_2$-EMA time constant;
    misfire under AdamW's different momentum structure. \\
Tiebreaker (validation $\pm 1\%$)
  & Optimizer-agnostic: loss-based. \\
LR exponent numerator $\beta_1$
  & Lion-specific form: Lion's $\beta_1$ is the momentum-gradient
    interpolation coefficient; other optimizers use $\beta_1$ differently. \\
LR exponent divisor $2$
  & Derived from the sign-based descent lemma (App.~\ref{app:exponent});
    the principle transfers but the exact coefficient requires
    per-optimizer re-derivation. \\
Compound-level integer state machine
  & Optimizer-agnostic: generic over votes (Algorithm~\ref{alg:lionvote}). \\
Cadence~$c$, max level~$L$
  & Optimizer-agnostic: generic hyperparameters. \\
\bottomrule
\end{tabular}
\end{table}

Two components are essentially Lion-specific: Vote~2's thresholds
(empirically demonstrated above) and the LR exponent's
numerator~$\beta_1$ (by Lion's specific use of~$\beta_1$ as the
momentum-gradient interpolation coefficient).  Three components are
optimizer-agnostic in input and require only Lion's sign-update
structure for one of the two derivations on their thresholds:
Vote~1, the tiebreaker, and the compound-level state machine.  The
divisor~$2$ principle transfers but the coefficient may not.  A
port to another momentum-based optimizer would therefore preserve
the architecture of the voting mechanism, recalibrate Vote~2 using
that optimizer's effective momentum time constant, and re-examine
the LR exponent.  We do not carry out such a port; this is left to
future work.

\paragraph{Single-seed limitation.}
These measurements (Vote~1, Vote~2, and Lion/AdamW cross-optimizer
comparisons above) are from a single seed.  The layer-type ordering
(norm more stable than attn/mlp) is consistent with the 8-seed
compound level trajectories in Table~\ref{tab:compound_vit} and the
structural argument in \S\ref{sec:analysis:layers}.  The Vote~2
dead-zone discrepancy between Lion and AdamW reflects a
single-seed observation of an effect
(${\sim}5{\times}$ ratio difference) whose magnitude far exceeds
plausible seed-to-seed variation and is therefore unlikely to reverse
under replication, though confirming this is left to future work.

\subsection{Partial Schedule Replacement}
\label{app:exp:schedule_replacement}

LionVote without any global schedule outperforms the
scheduled variant on WRN/C10 at cadence~8 by $0.18$~pp ($93.48\%$
vs.\ $93.30\%$), though this difference is not statistically
significant given the standard deviations
($\pm 0.39$ vs.\ $\pm 0.82$).  It falls short by $0.82$~pp on
ViT/C100 ($68.89\%$ vs.\ $69.71\%$;
Table~\ref{tab:ablation}).  The universally negative compound levels
implement a form of learning rate decay that overlaps in function with
cosine annealing.  On architectures where this overlap is
sufficient (WRN), the schedule becomes unnecessary; on ViTs, the
schedule provides additional value.

\subsection{Computational Cost Details}
\label{app:exp:computational_cost}

LionVote adds one in-place gradient accumulation per parameter per batch and
linear voting cost (one cosine similarity, one norm ratio, one integer
comparison per parameter) every $c$ epochs.  No second-order
information, meta-gradients, or additional forward/backward passes are
required.  Memory overhead beyond Lion is two tensors per parameter
(the gradient accumulator and the previous epoch-mean gradient
snapshot) plus two scalars (the compound level and the effective
learning rate).  The per-batch accumulation is a single in-place
addition per parameter (the same operation count as one line of Lion's
momentum update) and is negligible relative to the forward/backward
pass.  Voting fires once every $c$ epochs; for ViT-Tiny ($152$
parameter tensors), this is $152$ dot products and norms.  The
dominant overhead is memory: the two extra tensors triple Lion's
per-parameter optimizer state (${\sim}23$\,MB for Lion,
${\sim}69$\,MB for LionVote on ViT-Tiny).  For billion-parameter
models the $3{\times}$ increase is a deployment concern.

\paragraph{Hardware and total compute.}
All experiments were run on a single NVIDIA GPU (A100, T4, or L40S
depending on availability).  The 605 runs reported in this paper
required ${\sim}901$ GPU-hours in total:

\medskip
\begin{center}
\small
\begin{tabular}{l r r r}
\toprule
Configuration group & Runs & Total (hrs) & Avg per run (min) \\
\midrule
WRN-28-10 / CIFAR-10  & 285 & 417 & 88 \\
WRN-40-10 / CIFAR-100 &  88 & 176 & 120 \\
ViT-Tiny / CIFAR-10   &  88 &  72 & 49 \\
ViT-Tiny / CIFAR-100  & 112 & 146 & 78 \\
Controls (ViT/C100)    &  32 &  90 & 168 \\
\midrule
\textbf{Total}         & \textbf{605} & \textbf{901} & --- \\
\bottomrule
\end{tabular}
\end{center}
\medskip

No preliminary or failed experiments beyond those reported were
conducted.

\end{document}